\documentclass[journal]{IEEEtran}
\usepackage{amsmath}
\usepackage{amsthm}
\usepackage{amssymb}
\usepackage{algorithmicx}
\usepackage{algorithm}
\usepackage{algpseudocode}
\usepackage{array}
\usepackage{enumitem}
\usepackage{textcomp}
\usepackage{stfloats}
\usepackage{url}
\usepackage{verbatim}
\usepackage{graphicx}
\usepackage{cite}
\usepackage{amssymb}
\usepackage{extarrows}
\usepackage{multirow}
\usepackage{makecell}
\usepackage{booktabs}
\usepackage{color}
\usepackage{mathtools}
\usepackage{subfigure}
\usepackage{tikz}
\usepackage{hyperref}
\usepackage[capitalise]{cleveref}
\usepackage{booktabs,tabulary,lipsum}
\usepackage{dcolumn,tipa}
\newcolumntype{d}{D{.}{.}{6.5}}
\usepackage{siunitx}
\usepackage[usestackEOL]{stackengine}

\newtheorem{remark}{Remark}
\newtheorem{proposition}{Proposition}
\newtheorem{definition}{Definition}

\newcommand{\xun}[1]{{\color{black}#1}}

\begin{document}
\title{Moving-Horizon Estimation and Nonlinear Model Predictive Control of Cable-Driven Soft Manipulators}

\author{Lingxiao Xun, Haihong Li, Gang Zheng,~\IEEEmembership{Senior member,~IEEE}
	\thanks{}}

\markboth{}%
{Shell \MakeLowercase{\textit{et al.}}: A Sample Article Using IEEEtran.cls for IEEE Journals}


\maketitle

\begin{abstract}
Precise control of soft manipulators remains challenging due to the difficulty of developing accurate yet computationally tractable models for model-based estimation and control. Reduced Cosserat-rod models provide a physics-based and control-oriented description of soft-robot dynamics, offering an explicit alternative to purely data-driven input-output representations.
In this paper, we propose a moving-horizon estimation (MHE) and nonlinear model predictive control (NMPC) framework for cable-driven soft manipulators based on reduced Cosserat dynamics. A smooth cable-length-driven modeling formulation is developed by approximating the complementarity relationship between cable tension and cable slackness, enabling cable-length control without direct tension sensing. Based on this formulation, an MHE method is introduced to estimate the reduced state and reconstruct the manipulator configuration from end-effector pose measurements and cable-length information. An NMPC controller is then formulated to achieve task-space control under cable-length and cable-rate constraints.
The proposed framework is validated through numerical simulations and experiments. Simulation results demonstrate the effectiveness of the estimator and controller for pose and strain-related regulation on a multi-cable soft manipulator. Experimental results on a four-cable prototype further show that the proposed MHE-NMPC scheme can be implemented in real time and enables accurate end-effector position tracking through cable-length control.
\end{abstract}

\section{Introduction}
\subsection{State of the Art}
\IEEEPARstart{S}{oft} robotics has attracted considerable attention owing to its use of soft, elastic, or highly flexible materials, which enable robots to overcome some limitations of their rigid counterparts and undergo large deformations under external loads \cite{whitesides2018soft,chen2020design,hawkes2021hard}. It has great potential for applications in various fields such as exploration, actuation, and medical devices, and so on. However, there exist many degrees of freedom (DoFs) that correspond to infinite number of states of the system, which will lead to difficulty in modeling and controller design. Therefore, an efficient and high-performance controller of the soft robot requires the development of accurate mathematical models.
To achieve comparable levels of accuracy in the modeling and control of soft robots, various methodologies have been developed by researchers. These methods can be broadly classified into two primary categories: data-driven learning techniques and models derived from geometrical and mechanical analyses.

Learning-based approaches have proven effective in capturing the highly nonlinear dynamics of soft robots and implementing control strategies across different applications. For instance, the work in \cite{thuruthel2017learning} introduced a machine learning-based approach to develop a dynamic model for a soft manipulator, along with a trajectory optimization technique for predictive control. Similarly, \cite{gillespie2018learning} leveraged neural networks to construct nonlinear dynamic models of soft robots for model predictive control. In \cite{zheng2020robust}, a neural network-based control framework was introduced to address the control of soft robots, while \cite{tariverdi2021recurrent} proposed a recurrent neural network-based real-time dynamic model. Despite the potential of learning-based approaches, challenges remain, such as the difficulty of collecting sufficient and representative data, the limited generalization to new designs, and the risk of inadequate performance in real-world settings due to unaccounted external disturbances.

Conversely, model-based approaches rely on kinematic and dynamic equations to mathematically characterize the behavior of robotic systems. One of the most widely adopted methods involves modeling soft manipulator using curvature information. For example, the concept of constant curvature (CC) was introduced in \cite{hannan2003kinematics} to describe the bending motions observed in continuum robots. The piecewise constant curvature (PCC) method has also been extensively utilized for both kinematic and dynamic control, as demonstrated in \cite{della2020model} and \cite{della2020improved}, where a dynamic feedback control system was developed for trajectory tracking and surface following. However, the assumption of constant curvature is often too restrictive, particularly in scenarios involving external forces like contact and gravity, or more complex internal actuation mechanisms such as arbitrarily routed tendons. To address these limitations, alternative methods have been explored. For instance, \cite{mbakop2021inverse} proposed an inverse dynamic model based on the Euler-Bernoulli beam theory and parametric kinematic curves, which enabled real-time shape control of soft robots. Similarly, \cite{della2019exact} introduced a dynamic feedback controller for trunk-like soft robots, incorporating a dynamically consistent projector in synergistic space. 
The finite element method (FEM) is another tool used to develop kinematic and dynamic models for soft robots with complex shapes \cite{armanini2023soft,goury2018fast,duriez2013control}. However, the computational complexity of high-dimensional FEM models necessitates the development of model reduction techniques for real-time dynamic feedback control \cite{li2022equivalent}. The Cosserat model, a geometrically nonlinear generalization of traditional beam theories, is another powerful approach for modeling slender soft robots. In \cite{boyer2006macro}, the full dynamic Cosserat model was applied to hyper-redundant eel-type robots, while \cite{boyer2011macrocontinuous} explored its use in bio-inspired locomotion. Despite the strengths of these models, their reliance on nonlinear partial differential equations (PDEs) poses significant challenges for control design. Finite-dimensional approximations, as suggested in \cite{della2023model}, transform these PDEs into tractable ordinary differential equations (ODEs), which better facilitate the design of controllers. To address the limitations of simpler curvature models, extensions such as the piecewise constant strain (PCS) model \cite{renda2018discrete}, the piecewise linear strain (PLS) model \cite{li2023piecewise}, and the global variable strain (GVS) framework \cite{9057619} \cite{boyer2020dynamics} have been proposed.

Building upon the foundation of reduced Cosserat models, a variety of control outcomes have been achieved. For example, \cite{li2023discrete} describes the creation of a local controller utilizing PLS Cosserat dynamics for a slender soft manipulator, with its efficacy in tracking performance confirmed through numerous experiments. Nonetheless, this particular controller only effectively tracks slow signals within a limited workspace, owing to its omission of the model's velocity and acceleration impacts. Further research detailed in \cite{george2020first} introduces a reduction of the PCS dynamic model based on a first-order approximation, alongside the development of a closed-loop controller. Similarly, \cite{boyer2020dynamics} discusses the application of a dynamic computed torque controller for a planar, three-segment arm driven by tendons. Employing the same (GVS) parametrization, \cite{renda2022geometrically} proposed a Jacobian-based inverse kinematic controller for a 3D tendon-driven arm.

In recent years, several scholars have introduced model predictive control into the field of soft robotics. \cite{bruder2020data} implemented model predictive controllers (MPC) for a pneumatic soft robotic arm, utilizing a Koopman-based system identification method. Concurrently, \cite{thuruthel2017learning} devised an open-loop predictive control framework predicated on acquiring knowledge of the dynamic model. Complementing this approach, \cite{thuruthel2018model} and \cite{gillespie2018learning} developed closed-loop predictive control for soft robot, similarly grounded in dynamically learned models. However, thus far, the time-intensive nature of model predictive control has led researchers to predominantly utilize approximate linear models or learned models. This approach limits the scope for extending the models, such as by integrating additional variables or physical constraints, thereby constraining further advancements in the field. 
Furthermore, comprehensive experimental validations for many of these model-based control schemes have yet to be conducted. While substantial efforts have been made towards creating dynamic controllers for soft continuum robots, the focus has largely remained on simulations, with limited progression to experimental trials. Notably, experimental validations of control designs based on the reduced Cosserat dynamic model on actual robots are scant, which is critical for real-world scenarios that demand precise and swift positioning of soft robots.

Addressing the practical application of model-based controllers, a significant barrier remains in adapting these reduced models to the specific hardware configurations of soft robots. Typically, model-based state feedback controllers presume complete knowledge of the state vector, a challenging premise as state variables of soft robotic systems are generally difficult to measure accurately in practical settings. To achieve precise control over real robots, it is imperative to implement suitable state estimation techniques. This requirement presents challenges due to the infinite dimensionality characteristic of soft robots.
Several nonlinear state estimators, including moving horizon estimation \cite{michalska1995moving} and extended Kalman filter \cite{hartley2020contact}, have been widely employed for systems characterized by nonlinear dynamics and constrained state conditions. Noteworthy applications include a filtering method for shape and end-effector pose estimation of a snake robot \cite{tully2011shape}, real-time pose estimation techniques for tendon-driven continuum manipulators \cite{ataka2016real}, and methodologies for force and shape estimation in soft robotics \cite{navarro2020model}. Despite these developments, no existing state estimation techniques based on the full dynamic Cosserat model have been explored or validated for controlling cable-driven soft manipulators, marking a significant area for future research and experimentation.

\subsection{Contributions and outline of this work}
The primary aim of this research is to establish a comprehensive and generic model-based control framework for soft manipulators that encompasses both the development of an estimator and a controller. The specific contributions of this work are summarized as follows:
\begin{enumerate}
\item We propose a smooth cable-length-driven modeling formulation for cable-driven soft manipulators, together with the analytical Jacobians of the manipulator dynamics and cable-actuation constraints.

\item We develop an MHE method based on reduced Cosserat dynamics to estimate the manipulator state and reconstruct its shape from end-effector pose measurements and cable-length information.

\item We formulate an NMPC framework for task-space control of soft manipulators under cable-length and cable-rate constraints.

\item We validate the proposed MHE-NMPC framework through simulations and experiments, demonstrating pose regulation in simulation and real-time end-effector position tracking on a physical prototype.
\end{enumerate}

The structure of this paper is organized as follows:
In Section \ref{sec:2}, we introduce the research problem and outline the overall approach to solving it.
Section \ref{sec:3} revisits the Cosserat kinematics and defines the control objectives for the soft manipulator, including end-effector position and orientation control, as well as shape control.
In Section \ref{sec:4}, a cable-driven dynamic model specifically tailored for soft manipulators is presented.
Subsequently, in Section \ref{sec:45}, we derive the Jacobians of the dynamic residuals.
Section \ref{sec:5} presents the general nonlinear model predictive control framework, incorporating implicit dynamics and a smoothing method for cable-driven constraints, along with algorithms optimized for rapid computation.
Section \ref{sec:7} provides a series of numerical simulations to assess the performance of the proposed estimator and controller.
Finally, Section \ref{sec:8} focuses on validating the state estimator and evaluating the tracking performance of the control strategies applied to the end-effector through experimental studies. In the end,
the paper concludes in Section \ref{sec:9} with a summary of findings and potential avenues for future research.

For enhanced readability and ease of reference, all vector and matrix variables within this document are denoted in bold typeface, while scalar quantities are presented in standard typeface.
\section{Problem Statement}\label{sec:2}
Unlike the kinematics of rigid-body robots, the kinematics described above are continuous in space and involve an infinite number of degrees of freedom. This complexity poses challenges for subsequent control tasks. Consequently,
In this work, we represent the deformation of the soft manipulator using a finite number of degrees of freedom, following the strain-based modeling method proposed in many studies \cite{li2023piecewise}\cite{renda2018discrete}\cite{boyer2020dynamics}. 

This process is described as the discrete Cosserat dynamic model in \cite{li2023piecewise} as the ordinary differential equation (ODE). Building on the Cosserat model, we aim to develop a model-based estimation and control framework for the following systems of soft manipulator.
\xun{
\begin{equation}\label{Lagrangian_system}
	\boldsymbol{f}(\ddot{\boldsymbol{q}},\dot{\boldsymbol{q}},\boldsymbol{q},\boldsymbol{u}) = \boldsymbol{0}
\end{equation}
\begin{equation}\label{Lagrangian_systemz}	\boldsymbol{\alpha}=\boldsymbol{\gamma}(\dot{\boldsymbol{q}},\boldsymbol{q}) 
\end{equation}
\begin{equation}\label{Lagrangian_systemy}	\boldsymbol{y}=\boldsymbol{h}(\dot{\boldsymbol{q}},\boldsymbol{q}) 
\end{equation}
with $\boldsymbol{q}$ representing the general coordinates of the soft manipulator and $\boldsymbol{u}$ denoting the inputs. \eqref{Lagrangian_system} represents the Cosserat dynamics of soft manipulator. The measurements and outputs to be controlled, denoted by $\boldsymbol{\alpha}$ and $\boldsymbol{y}$, are functions of $\boldsymbol{q}$ and $\dot{\boldsymbol{q}}$ through Cosserat kinematics \eqref{Lagrangian_systemz} and \eqref{Lagrangian_systemy}.}

These systems, however, exhibit high-dimensional, nonlinear, and often underactuated dynamics, necessitating advanced estimation and control strategies for precise motion regulation. The specific control objectives depend on the system’s characteristics, available sensor measurements, and desired performance. For a soft manipulator, potential objectives may include regulating position, orientation, strain, or acceleration, as well as minimizing energy consumption or achieving a specified control precision.

To accommodate these diverse objectives, this work proposes a comprehensive control framework capable of addressing various control goals. Within this framework, we focus on two specific applications: controlling the soft manipulator’s end-effector position and regulating strain. To achieve this, we first reformulate the general estimation and control objectives as an output tracking problem, leveraging a predictive estimation and control perspective.
Specifically, defining the estimation time horizon as $t_e$ and the control time horizon as $t_c$, we adopt a nonlinear model predictive estimation and control approach that simultaneously:

\begin{itemize}
	\item {Estimates} the system state via moving horizon estimation over a retrospective time window $[t - t_{\mathrm{est}}, t]$.
	\item {Optimizes} control actions via model predictive control over a prospective time window $[t, t + t_{\mathrm{ctrl}}]$.
\end{itemize}

Denoting $\boldsymbol{x}=[\boldsymbol{q}^\top,\dot{\boldsymbol{q}}^\top]^\top$ as the state of the system. 
At each time instant $t$, the MHE-NMPC problem is formulated as, $\tau \in [t - t_{\mathrm{est}}, t + t_{\mathrm{ctrl}}]$:
\begin{equation}\label{optsm}
	\begin{aligned}
		\min_{\substack{\boldsymbol{x}(\tau), \boldsymbol{u}(\tau)}} 
		 \quad \int_{t - t_{\mathrm{est}}}^{t}&\mathcal{J}_{\mathrm{est}}(\boldsymbol{\alpha}) \ \mathrm{d}\tau + \int_{t}^{t + t_{\mathrm{ctrl}}}\mathcal{J}_{\mathrm{ctrl}}(\boldsymbol{y}) \ \mathrm{d}\tau\\
		\mathrm{subject \ to} \quad  \quad
		& C_1: \quad  f(\dot{\boldsymbol{x}},\boldsymbol{x}, \boldsymbol{u})=\boldsymbol{0},\\
        & \xun{C_2: \quad  \boldsymbol{\alpha}=\boldsymbol{\gamma} (\boldsymbol{x}), \quad\boldsymbol{y}=\boldsymbol{h}(\boldsymbol{x}),}\\
		& C_3: \quad (\boldsymbol{x},\boldsymbol{u}) \in \mathcal{W},\\
	\end{aligned}
\end{equation}
(\ref{optsm}) includes the optimization objective function and constraints, which are distributed as follows.
\begin{itemize}
	\xun{\item $\mathcal{J}_{\mathrm{est}}$ and $\mathcal{J}_{\mathrm{ctrl}}$ represent the objectives of estimation and control respectively.}
	\item $C_1$ represents the constraint of nonlinear dynamics of the soft manipulator.
    \xun{\item $C_2$ represents the Cosserat kinematics of measurements $\boldsymbol{\alpha}$ and outputs $\boldsymbol{y}$.}
	\item $C_3$ represents the actuation constraints, depending on the working principle of actuator, such as pneumatic, cable driven, etc.
\end{itemize}

The introduction of actuation constraints serves to address the physical limitations that arise in various practical contexts and to integrate the intrinsic models of actuators. This includes, for example, the physical relationship between chamber volume and pressure in pneumatic actuation, the correlation between cable length and tension in cable-driven systems, and the models pertinent to dielectric elastomers. In this study, we will particularly focus on exploring the cable-driven model, which will be elaborately discussed in Section \ref{sec:4}.

Clearly, to accommodate different control scenarios, this framework can be effectively applied to manage the control of multiple objectives simultaneously by consolidating the variables into a single unified vector.

Building upon the optimization framework for estimation and control, this work primarily addresses four key issues:
\begin{enumerate}
	\item Defining objective functions for estimation and control, along with their gradients, as discussed in Section \ref{sec:3}.
	\item Deriving the dynamic model with appropriate actuation constraints to fit the optimization-driven framework, as presented in Section \ref{sec:4}.
	\item Formulating the analytical Jacobians of the dynamics to facilitate efficient optimization, which will be detailed in Section \ref{sec:45}.
	\item Constructing the final MHE-NMPC problem and selecting suitable optimization algorithms for its solution, as outlined in Section \ref{sec:5}.
\end{enumerate}

In the following sections, we will address these issues one by one.
\section{Estimation and Control Objectives}\label{sec:3}
In this section, we introduce a reduced Cosserat-rod model for soft manipulators, which provides a compact yet accurate representation of their continuous deformation. Building on this reduced model, we formulate control and estimation objectives and derive their gradients with respect to the reduced coordinates, laying the foundation for efficient trajectory and shape regulation.

\subsection{Reduced model of soft manipulator}
In the Cosserat framework, the soft manipulator is considered as a set of rigid cross-sections along its centerline, see Fig. \ref{introduction}. 
\begin{definition}[Homogeneous transformation matrix]
The homogeneous transformation matrix for any cross section along the soft manipulator can be defined as $$\forall s\in [0,L] ,\ \boldsymbol{g}(s,t)=\begin{bmatrix}
	\boldsymbol{R}(s,t)&\boldsymbol{p}(s,t)\\
	\boldsymbol{0}&1
\end{bmatrix}\in SE(3)$$
where $\boldsymbol{p}(s,t)\in\mathbb{R}^3$ is the position vector and $\boldsymbol{R}(s,t)\in SO(3)$ represents an orthonormal rotation matrix. $L$ is the total arc length of the soft manipulator.
\end{definition}
\begin{figure}[t]
	\centering
	\includegraphics[width=0.45\textwidth]{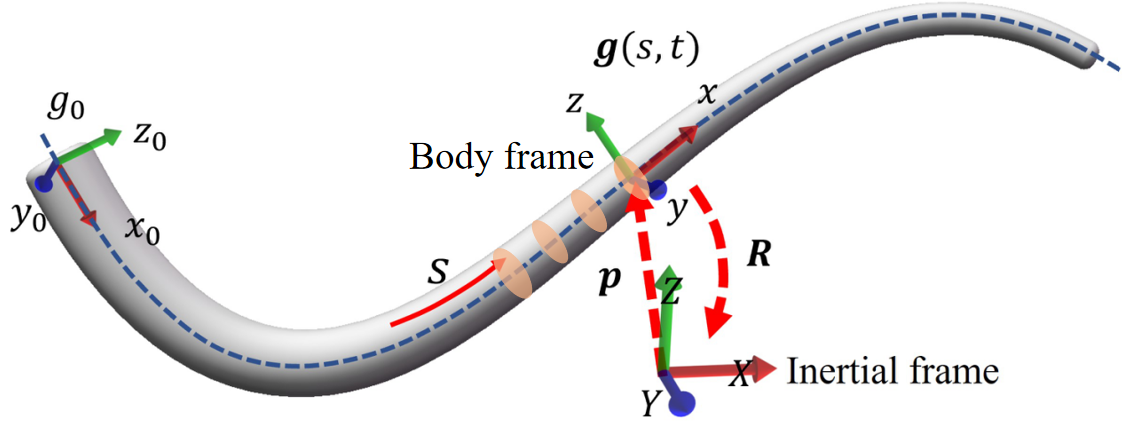}
	\caption{Schematic diagram of the soft manipulator.}
	\label{introduction}
\end{figure}

Subsequently, the strain and velocity are defined by the tangent space of the homogeneous transformation matrix.
\begin{definition}[Strain and velocity of Cosserat rod]
	The strain and velocity in the body frame can be regarded as the left-trivialized tangent space of the homogeneous transformation matrix w.r.t. space and time, i.e.,
	$$\begin{aligned}
		\hat{\boldsymbol{\xi}}(s,t)&=\boldsymbol{g}^{-1}\boldsymbol{g}^{\prime}\in \mathfrak{se}(3)\simeq \mathbb{R}^6,\\  \hat{\boldsymbol{\eta}}(s,t)&=\boldsymbol{g}^{-1}\dot{\boldsymbol{g}}\in \mathfrak{se}(3)\simeq \mathbb{R}^6
	\end{aligned}$$

\end{definition}
For simplicity, we use ${(\cdot)^\prime}$ to denote the partial derivative w.r.t space $\partial/\partial s$ and $\dot{(\cdot)}$ to denote the partial derivative w.r.t. time $\partial/\partial t$. 

To represent the state of a soft manipulator with a finite set of variables, we parameterize its strain field using a reduced set of degrees of freedom \cite{mathew2025reduced}:
\begin{equation}\label{para}
\boldsymbol{\xi}(s,t) = \boldsymbol{\xi}_0 + \boldsymbol{\Phi}(s)\boldsymbol{q}(t),
\end{equation}
where $\boldsymbol{q}(t)\in\mathbb{R}^N$ denotes the generalized coordinates and $\boldsymbol{\Phi}(s)\in\mathbb{R}^{6\times N}$ represents the shape function matrix. Under this parameterization, the corresponding kinematics mapping is given by
\begin{equation}\label{gexp}
\boldsymbol{g}(s,t) = \mathrm{exp}\bigl(\Omega (s)\bigr),
\end{equation}
\begin{equation}\label{etajaco}
\boldsymbol{\eta}(s,t) = \boldsymbol{J}(s,t)\dot{\boldsymbol{q}},
\end{equation}
where $\Omega (s)$ denotes the Lie group integration via Magnus expansion \cite{mathew2025reduced}. $\boldsymbol{J}(s,t)$ is the kinematic Jacobian \cite{mathew2025reduced}, deduced as
\begin{equation}\label{jaco}
\boldsymbol{J}(s,t) = \mathrm{Ad}^{-1}{\boldsymbol g(s,t)} \int_0^s \mathrm{Ad}_{\boldsymbol g(x,t)} \boldsymbol{\Phi}(x) \, \mathrm{d}x.
\end{equation}

Based on the reduced parameterization of the soft manipulator, we now formulate the control and estimation objectives and derive their gradients with respect to the reduced coordinates $\boldsymbol{q}$.

\subsection{Control objective and its gradient}\label{secposer}

Within the reduced Cosserat rod framework, this subsection formulates the control and estimation objectives, together with their analytical gradients, \xun{which are essential for efficient optimization.} We first address the pose control problem, followed by the shape estimation.

\subsubsection{Pose control}
Let the controlled output $\boldsymbol{y}$ be the pose $\boldsymbol{g}(s,t)$, typically the end-effector pose $\boldsymbol{g}(L,t)$. Giving the desired orientation $\boldsymbol{R}_r$ and position $\boldsymbol{p}_r$, we define the following control error:
\begin{equation}\label{jg}
\mathcal{J}_{\boldsymbol g}(\boldsymbol g) := \frac{a}{2}\,\mathrm{tr}(\mathrm{I}-\boldsymbol R_r^\top \boldsymbol R) + \frac{b}{2}\,\Vert \boldsymbol p - \boldsymbol p_r \Vert_2^2,
\end{equation}
where the first term represents the orientation error and the second term represents the position error, and $a$ and $b$ are the corresponding weighting factors. Based on \eqref{gexp}, the output function in~\eqref{Lagrangian_systemy} is  $\boldsymbol{y} = \boldsymbol{h}(\boldsymbol{q}) : = \exp\bigl(\Omega(L)\bigr)$. Therefore, the control objective function w.r.t $\boldsymbol{q}$ is
\begin{equation}\label{jctrl}
    \mathcal{J}_{\mathrm{ctrl}}(\boldsymbol q) := \mathcal{J}_{\boldsymbol g}\circ\boldsymbol{h}
\end{equation}
Next, we derive its gradient with respect to the reduced degrees of freedom $\boldsymbol{q}$.

\begin{proposition}\label{prop1}
For a soft manipulator whose strain field is approximated by a reduced coordinate vector $\boldsymbol{q}$ via the parametrization in \eqref{para}, the gradient of $\mathcal{J}_{\mathrm{ctrl}}$ in \eqref{jctrl} with respect to $\boldsymbol{q}$ is
\begin{equation}\label{jacog}
\nabla{\boldsymbol q} \, (\mathcal{J}_{\boldsymbol g}\circ\boldsymbol{h}) = \boldsymbol J^\top \boldsymbol \zeta,
\end{equation}
where
$\boldsymbol{J}$ denotes the kinematic Jacobian of the end-effector defined in \eqref{jaco} and
\begin{equation}
\boldsymbol \zeta =
\begin{bmatrix}
-a\,\bigl(\mathrm{skew}(\boldsymbol R_r^\top \boldsymbol R)^{\vee}\bigr)^\top &
b\,(\boldsymbol p - \boldsymbol p_r)^\top \boldsymbol R
\end{bmatrix}^\top.
\end{equation}
Here, $\mathrm{skew}(\boldsymbol X) := (\boldsymbol X - \boldsymbol X^\top)/2$ denotes the skew-symmetric component of the matrix $\boldsymbol X$.
\end{proposition}

\begin{proof}  
Assuming the target pose is constant, taking the time derivative of the objective function \eqref{jg} yields  
\[
\begin{aligned}
\dot{\mathcal{J}}_{\boldsymbol{g}} &= -\frac{a}{2}\,\mathrm{tr}(\boldsymbol R_r^\top \dot{\boldsymbol R}) + b\,(\boldsymbol p - \boldsymbol p_r)^\top \dot{\boldsymbol p} \\
&= -\frac{a}{2}\,\mathrm{tr}(\boldsymbol R_r^\top \boldsymbol R \tilde{\boldsymbol{\omega}}) + b\,(\boldsymbol p - \boldsymbol p_r)^\top \boldsymbol R \boldsymbol v,
\end{aligned}
\]  
where $\boldsymbol{\omega}$ and $\boldsymbol{v}$ are the angular and linear velocities of the end-effector expressed in the body frame.  

Using the identity that for any $\boldsymbol A \in \mathbb{R}^{3\times3}$ and $\boldsymbol x \in \mathbb{R}^3$,  
\[
\mathrm{tr}(\boldsymbol A \tilde{\boldsymbol x}) = -2\,\mathrm{skew}(\boldsymbol A)^{\vee\top} \boldsymbol x,
\]  
we obtain  
\[
\dot{\mathcal{J}}_{\boldsymbol g} = -a\,\bigl(\mathrm{skew}(\boldsymbol R_r^\top \boldsymbol R)^{\vee}\bigr)^\top\boldsymbol{\omega} + b\,(\boldsymbol p - \boldsymbol p_r)^\top \boldsymbol R \boldsymbol v.
\]  

Subsequently, denoting the body-frame velocity twist as $\boldsymbol \eta = [\boldsymbol \omega^\top \ \boldsymbol v^\top]^\top$ and applying the kinematics mapping \eqref{etajaco}, it follows that  
\[
\dot{\mathcal{J}}_{\boldsymbol g} = \boldsymbol \zeta^\top \boldsymbol J \dot{\boldsymbol q}.
\]  
Hence, the gradient with respect to the reduced coordinates is  
\[
\nabla{\boldsymbol q} \, (\mathcal{J}_{\boldsymbol g}\circ\boldsymbol{h}) = \boldsymbol J^\top \boldsymbol \zeta.
\] 
\end{proof}

\subsubsection{Shape estimation}
\xun{For estimation, the most readily measurable quantities are typically the position and orientation of a point on the soft manipulator; accordingly, the estimation objective shares the same structure as the pose control objective. Let the measured output $\boldsymbol{\alpha}$ be the end-effector pose $\boldsymbol{g}(L,t)$. The measurements function in~\eqref{Lagrangian_systemz} is therefore $\boldsymbol{\alpha} = \boldsymbol{\gamma}(\boldsymbol{q}) := \exp\bigl(\Omega(L)\bigr)$. The estimation objective is defined as
\begin{equation}\label{jest}
\mathcal{J}_{\mathrm{est}}(\boldsymbol q) = \mathcal{J}_{\boldsymbol g} \circ \boldsymbol{\gamma}.
\end{equation}
The corresponding gradient is derived analogously to~\eqref{jacog}:}
\begin{equation}\label{jacog2}
\nabla{\boldsymbol q} \, (\mathcal{J}_{\boldsymbol g}\circ\boldsymbol{\gamma}) = \boldsymbol J^\top \boldsymbol \zeta,
\end{equation}

As mentioned in the previous section, this paper explores the scenario where cables are utilized as the actuation mechanism for soft manipulators to achieve the specified control objectives. The subsequent section will delve into the modeling of cable actuation and its integration with the reduced dynamic model.
\section{Dynamics with Actuation Constraints}\label{sec:4}
Common actuation methods for soft manipulator include cable-driven\cite{renda2014dynamic}, pneumatic\cite{marchese2016design}, and dielectric elastomer actuation\cite{gu2017survey}. Among these, the cable-driven mechanism is relatively simple, offering strong actuation force and rapid response speeds, making it the most widely applied in soft robotics. Therefore, in this article, we focus on exploring models that utilize cables for actuation, aiming to enhance the application of this actuation method for the control of soft manipulator.

For cable-driven soft manipulators, force control presents considerable challenges in engineering practice due to the need for integrating force sensors and implementing complex algorithms to manage fluctuations in cable tension. To circumvent these issues, controlling the length of the cables has proven to be an effective alternative, greatly simplifying application in practical settings. Nonetheless, traditional force control algorithms often encounter difficulties transitioning to cable length control in soft manipulators. To address this, we will start with geometric calculations pertaining to cable length, subsequently introduce a cable length constraint model, and finally implement this model in cable length control strategies.
\begin{figure}[ht]
	\centering
	\includegraphics[width=0.45\textwidth]{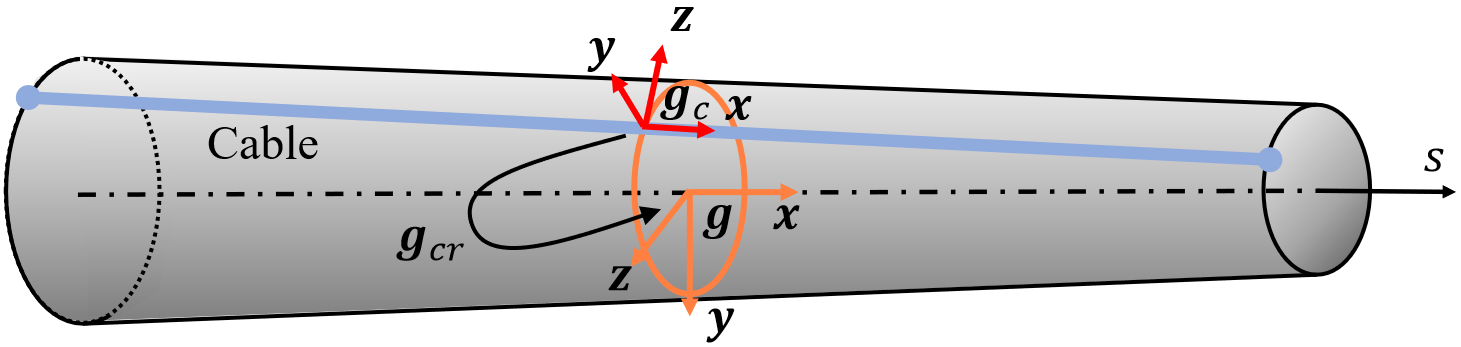}
	\caption{Initial configuration of soft manipulator and the cable path.}
	\label{cable1}
\end{figure}
\subsubsection{Inner routing path of the cable}
As illustrated in Fig.~\ref{cable1}, each cable passes through an internal routing path within the soft manipulator and is anchored at one of its cross sections, typically at the free end. The cable path runs approximately parallel to the surface of the soft body. At each cross section, its position and orientation can be described by a configuration matrix $\boldsymbol{g}_c$ \xun{w.r.t. the global frame.}
Subsequently, we can define a relative transformation matrix $\boldsymbol{g}_{cr}$ representing the transformation from this frame associated with the cable routing path to the body frame attached to the disc along the centerline of the soft manipulator, i.e., $\boldsymbol{g}_c=\boldsymbol{g}\boldsymbol{g}_{cr}$. This transformation determines the how the cable passes through the soft manipulator at the initial (natural) state, and is defined as follows:
$$\boldsymbol{g}_{cr}(s)=\begin{bmatrix}
	\boldsymbol{R}_{cr}(s) & \boldsymbol{d}_{cr}(s)\\
	\boldsymbol{0}&1
\end{bmatrix}.$$ 
Note that this matrix is fixed once the soft manipulator is designed and remains constant throughout its deformation. Let the strain along the cable path with respect to the local frame associated with cable path be denoted by $\boldsymbol{\xi}_{c}$, i.e., $\hat{\boldsymbol{\xi}}_{c}=\boldsymbol{g}_c^{-1}\boldsymbol{g}_c^\prime$. \xun{Using $\boldsymbol{g}_c=\boldsymbol{g}\boldsymbol{g}_{cr}$, the strain $\boldsymbol{\xi}_{c}$} can be directly derived from the strain of the soft manipulator as follows:
$$
\boldsymbol{\xi}_{c}=\mathrm{Ad}_{\boldsymbol{g}_{cr}}\boldsymbol{\xi}+\boldsymbol{\xi}_{cr}
$$ 
with $\boldsymbol{\xi}_{cr}=(\boldsymbol{g}_{cr}^{-1}\boldsymbol{g}_{cr}^{\prime})^\vee$.
The linear strain component of $\boldsymbol{\xi}_{c}$ represents the rate of change in the length of the cable routing path w.r.t. the arc length $s$. 
Subsequently, the length of the routing path yields
\begin{equation}\label{ccc}
l_{c}(\boldsymbol{q})=\int_{0}^{L}\Vert \mathbf{D}\boldsymbol{\xi}_{c}\Vert\, \mathrm{d}s
\end{equation}
where $\mathbf{D}$ is a selection matrix that extracts the linear strain from $\boldsymbol{\xi}_c$, given by $\mathbf{D} = [\boldsymbol{0}_{3} \ \mathbf{I}_3]$.

Obviously, the length of the routing path is a function that depends on the generalized strain $\boldsymbol{q}$ of the soft manipulator and its gradient can be computed analytically. Using chain rule, we can get:

\begin{equation}\label{ccsc}
	\nabla_{\boldsymbol{q}}\,l_c  = 
	\int_{0}^{L} \boldsymbol{\Phi}^\top\mathrm{Ad}^\top_{\boldsymbol{g}_{cr}}\mathbf{D}^\top\boldsymbol{n}\, \mathrm{d}s 
\end{equation}
%
where $\boldsymbol{n}$ is the unit direction vector of the cable routing path, i.e., $\boldsymbol{n}=\mathbf{D}\boldsymbol{\xi}_{c}/\Vert \mathbf{D}\boldsymbol{\xi}_{c}\Vert$.
It is worth emphasizing that \eqref{ccsc} is equivalent to the transpose of generalized actuation matrix from the perspective of the Lagrangian dynamics, i.e., the general force introduced by the cable tension $T$ is  \begin{equation}\label{Fact}	\boldsymbol{F}_{act}=\nabla_{\boldsymbol{q}}\,l_c\, T=\int_{0}^{L} \boldsymbol{\Phi}^\top\mathrm{Ad}^\top_{\boldsymbol{g}_{cr}}\mathbf{D}^\top\boldsymbol{n}\, \mathrm{d}s\,T
\end{equation}
Subsequently, the Jacobian of $\boldsymbol{F}_{act}$ can be computed by differentiating~\eqref{Fact}. Noting that $\boldsymbol{\Phi}^\top \mathrm{Ad}^\top_{\boldsymbol{g}_{cr}} \mathbf{D}^\top$ is invariant with respect to $\boldsymbol{q}$, taking the derivative of~\eqref{Fact} yields:
\begin{equation}
		\frac{\partial \boldsymbol{F}_{act}}{\partial \boldsymbol{q}}=\int_{0}^{L}\boldsymbol{\Phi}^\top\mathrm{Ad}^\top_{\boldsymbol{g}_{cr}}\mathbf{D}^\top\frac{\mathbf{I}-\boldsymbol{n}\boldsymbol{n}^\top}{\Vert\mathbf{D}\boldsymbol{\xi}_{c} \Vert}\mathbf{D}\mathrm{Ad}_{\boldsymbol{g}_{cr}}\boldsymbol{\Phi}\,\mathrm{d}s\,T
\end{equation}
In general, soft manipulators are actuated by multiple tendons or cables. Suppose the manipulator is driven by \(m\) cables. For notational convenience, we define the aggregated generalized actuation force as $\boldsymbol{\Phi}^\top\mathcal{A}\boldsymbol{u}$, where $\boldsymbol{u}$ is the vector collecting all cable tensions from $T_1$ to $T_m$, and $\mathcal{A}$ denotes the corresponding input matrix:
\[\boldsymbol{u} =
\begin{bmatrix}
T_1 & T_2 & \dots & T_m
\end{bmatrix}^\top,\]
\[\mathcal{A} =
\begin{bmatrix}
\mathrm{Ad}^\top_{\boldsymbol{g}_{cr,1}}\mathbf{D}^\top\boldsymbol{n}_1 &
\mathrm{Ad}^\top_{\boldsymbol{g}_{cr,2}}\mathbf{D}^\top\boldsymbol{n}_2 &
\dots &
\mathrm{Ad}^\top_{\boldsymbol{g}_{cr,m}}\mathbf{D}^\top\boldsymbol{n}_m
\end{bmatrix}.
\]
\subsubsection{Dynamics of cable driven manipulator}
The Cosserat rod based dynamics of soft manipulator has been represented in many previous studies \cite{li2023piecewise}\cite{boyer2020dynamics}, we recall it here for readers better understanding our work, as we will deduce its gradients in the next section. The residual of dynamics \eqref{Lagrangian_system} is detailed as below:
\begin{equation}\label{dy23}
		\boldsymbol{f}(\ddot{\boldsymbol{q}},\dot{\boldsymbol{q}},\boldsymbol{q},\boldsymbol{u})=\int_{0}^{L}(\boldsymbol{J}^\top\boldsymbol{\Lambda}_{ine}+ \boldsymbol{\Phi}^\top\boldsymbol{\Lambda}_{int}) \, \mathrm{d}s
\end{equation}
where $\boldsymbol{\Lambda}_{ine}$ denotes the inertial force plus the external force $\boldsymbol{\Lambda}_{ext}$ applied on the soft manipulator, given by:
\begin{equation}\label{intertial}
	\boldsymbol{\Lambda}_{ine} = \mathcal{M}\dot{\boldsymbol{\eta}}-\mathrm{ad}^\top_{\boldsymbol{\eta}}\mathcal{M}\boldsymbol{\eta}+\boldsymbol{\Lambda}_{ext}
\end{equation}
$\boldsymbol{\Lambda}_{int}$ denotes the internal force, containing the elastic internal force and actuation force of cable:
\begin{equation}\label{internal}
	\boldsymbol{\Lambda}_{int} = \mathcal{K}\boldsymbol{\Phi}\boldsymbol{q}+\mathcal{D}\boldsymbol{\eta}+ \mathcal{A}\boldsymbol{u}
\end{equation}
\xun{where $\mathcal{K}$ and $\mathcal{D}$ denote the material stiffness and viscous tensor respectively \cite{renda2018discrete}. }

The above equations \eqref{dy23}–\eqref{internal} describe the dynamics of the soft manipulator under tendon-tension actuation. In practical systems, however, cable lengths are typically easier to measure and control than cable tensions. When the system inputs are defined in terms of cable lengths, it becomes necessary to incorporate these lengths into the manipulator’s dynamics. In this tendon-length–driven scenario, and cable-length constraints must be imposed. Moreover, since cables can only exert pulling forces and cannot push, the cable tensions are constrained to be non-negative, which can be expressed mathematically as a complementarity condition.
\subsubsection{Complementarity condition of cable}
since the cables are inextensible and can only provide tension (pulling force) rather than compression (pushing force), we need to add constraints to the dynamic system.
\begin{figure}[h]
	\centering
	\includegraphics[width=0.49\textwidth]{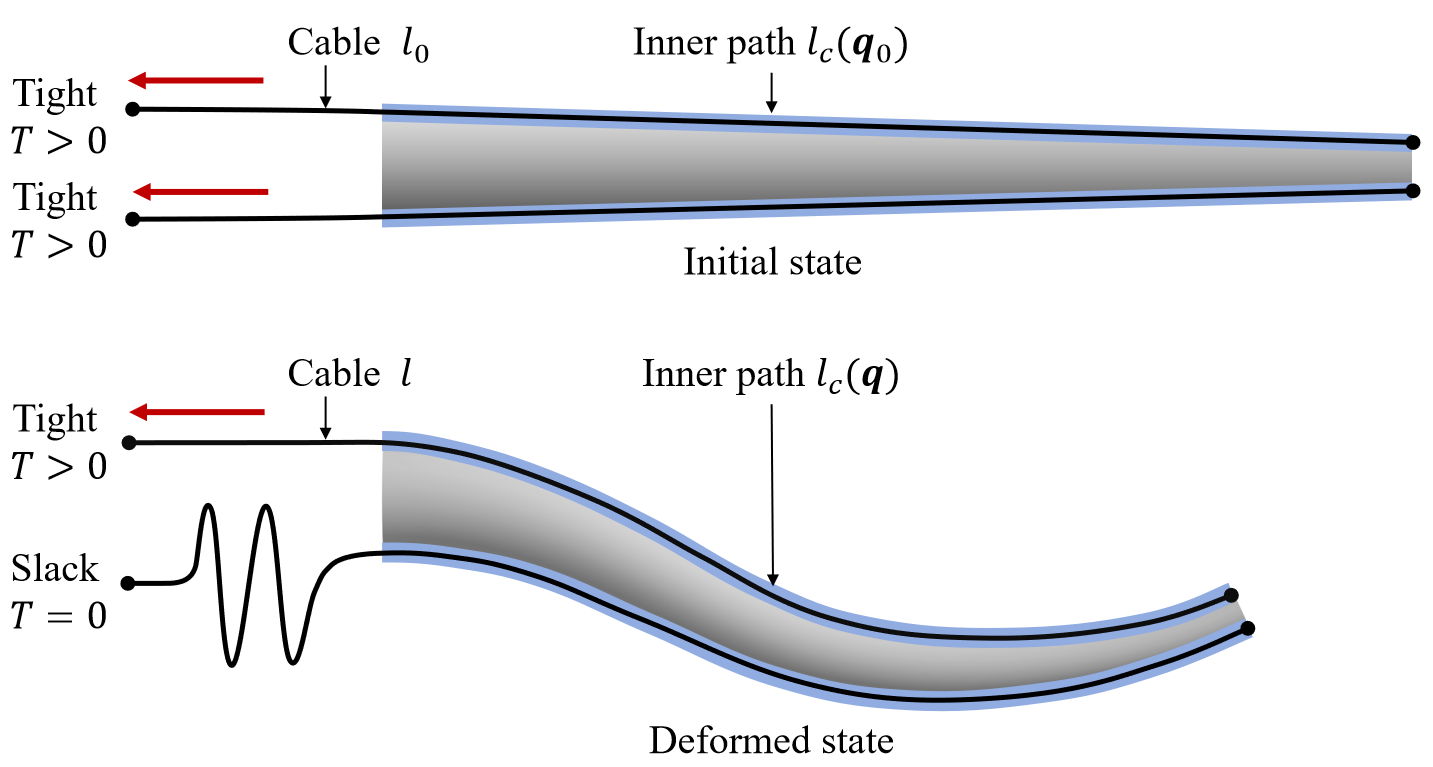}
	\caption{Complementarity condition of cable.}
	\label{cable2}
\end{figure} As shown in \cref{cable2}, during the operation of a cable-driven soft manipulator, two scenarios can occur: the cable tightens, with the tension greater than zero, and the cable slackens, where the tension equals zero. We describe the constraints associated with these two scenarios using the following nonlinear complementarity condition:
\begin{equation}\label{lcp0}
0\leq T_i \perp \delta_i\geq 0, \quad\forall i\in{1,\dots,m},
\end{equation}
where the symbol ``$\perp$'' indicates that the two nonnegative quantities are complementary, i.e., $T_i\delta_i=0$. Here, $T_i$ denotes the cable tension, and $\delta_i$ represents the cable-length gap. Specifically, $\delta_i$ is defined as the difference between the effective cable length and the length of the internal cable path inside the soft manipulator:
$$
\delta_i = l_i-l_{ci}(\boldsymbol{q}).
$$
\xun{The effective cable length $l_i$ is obtained by subtracting the initial external cable length offset from the total cable length. This offset corresponds to the length of the external cable segment between the cable entry point of the soft manipulator and the motor output when the cable is taut.} This complementarity condition enforces that the compliant cable can transmit tension only and cannot provide compression, and that there is no relative displacement at the contact point between the end of the soft manipulator and the cable tip.

For notational convenience, we rewrite \eqref{lcp0} in vector form as
\begin{equation}\label{lcp}
	\boldsymbol{0}\leq \boldsymbol{u} \perp \boldsymbol{l}-\boldsymbol{l}_{c}(\boldsymbol{q})\geq \boldsymbol{0},
\end{equation}
where
\[
\boldsymbol{l} =
\begin{bmatrix}
l_1 & l_2 & \dots & l_m
\end{bmatrix}^\top,\quad
\boldsymbol{l}_c =
\begin{bmatrix}
l_{c1} & l_{c2} & \dots & l_{cm}
\end{bmatrix}^\top
\]
collect the cable lengths and their corresponding inner routing lengths, respectively.

In this section, we have derived the dynamics \eqref{dy23} of the soft manipulator subject to cable-length constraints \eqref{lcp}. In the following section, we will derive the corresponding dynamic Jacobian, which will be employed within the subsequent NMPC optimization framework.

\section{Jacobian of the Dynamics}\label{sec:45}
In this section, we will derive the analytical Jacobian matrix of the residual of the dynamics of the soft manipulator with respect to \( \boldsymbol{q} \), \( \dot{\boldsymbol{q}} \), and \( \ddot{\boldsymbol{q}} \). This Jacobian is used to solve the optimization problem presented in Section~\ref{sec:5} and is critical for achieving computational efficiency.

To derive the Jacobian of the dynamics in \eqref{dy23}-\eqref{internal}, we must evaluate the following groups of partial derivatives:

\begin{equation}\label{sss}
	\frac{\partial \boldsymbol{J}^\top}{\partial \boldsymbol{q}} \mathbf{d}, \quad
	\frac{\partial \boldsymbol{\eta}}{\partial \boldsymbol{q}}, \quad
	\frac{\partial \boldsymbol{\eta}}{\partial \dot{\boldsymbol{q}}}, \quad
	\frac{\partial \dot{\boldsymbol{\eta}}}{\partial \boldsymbol{q}}, \quad
    \frac{\partial \dot{\boldsymbol{\eta}}}{\partial \dot{\boldsymbol{q}}}, \quad
	\frac{\partial \dot{\boldsymbol{\eta}}}{\partial \ddot{\boldsymbol{q}}}
\end{equation}
where the vector \( \mathbf{d} \) is an arbitrary element in \( \mathbb{R}^6 \).  

The computation process is divided into four stages. First, we recall the chain rule identities that establish the relationships among \( \boldsymbol{\eta} \), \( \dot{\boldsymbol{\eta}} \), and the kinematic Jacobian. Second, we derive closed-form expressions for the kinematic Jacobian \( \boldsymbol{J} \) and its time derivative \( \dot{\boldsymbol{J}} \). Third, we compute the partial derivatives of the adjoint maps that appear in these expressions. Finally, we consolidate these results into a concise computational procedure for the Jacobian matrix of the residual of the dynamics. The complete sequence of steps is detailed as the following subsections.

\subsection{Derivatives of $\boldsymbol{\eta}$ and its time derivative}
Recalling $\boldsymbol{\eta}=\boldsymbol{J}\dot{\boldsymbol{q}}$, we can immediately deduce the following gradients: 
\begin{align}
	\frac{\partial \boldsymbol{\eta}}{\partial\boldsymbol{q}}&=\frac{\partial\boldsymbol{J}}{\partial\boldsymbol{q}}\dot{\boldsymbol{q}},
	&
	\frac{\partial \boldsymbol{\eta}}{\partial\dot{\boldsymbol{q}}}&= \boldsymbol{J},&&\label{eta1}\\
	\frac{\partial \dot{\boldsymbol{\eta}}}{\partial\boldsymbol{q}}
	&=\frac{\partial\dot{\boldsymbol{J}}}{\partial\boldsymbol{q}}\dot{\boldsymbol{q}}
	+\frac{\partial\boldsymbol{J}}{\partial\boldsymbol{q}}\ddot{\boldsymbol{q}},
	&
	\frac{\partial \dot{\boldsymbol{\eta}}}{\partial\dot{\boldsymbol{q}}}
	&=\dot{\boldsymbol{J}}+\frac{\partial\dot{\boldsymbol{J}}}{\partial\dot{\boldsymbol{q}}}\dot{\boldsymbol{q}},	&\frac{\partial \dot{\boldsymbol{\eta}}}{\partial\ddot{\boldsymbol{q}}}
	&=\boldsymbol{J}\label{eta2}
\end{align}
The relations above expose the Jacobian $\boldsymbol{J}$ as the central object: once its gradient is known, {all} remaining derivatives follow algebraically.  We therefore turn next to explicit formulae for $\boldsymbol{J}$ and~$\dot{\boldsymbol{J}}$. Following the formulation in \eqref{jaco}, the Jacobian and its time derivative are:
\begin{align}
	\boldsymbol{J}&= \mathrm{Ad}_{\boldsymbol{g}_s}^{-1}
	\int_{0}^{s} \mathrm{Ad}_{\boldsymbol{g}}\boldsymbol{\Phi}\,\mathrm{d}x,\label{jaco25}\\
	\dot{\boldsymbol{J}}&= \mathrm{Ad}_{\boldsymbol{g}_s}^{-1}
	\Bigl( \int_{0}^{s} \dot{\mathrm{Ad}_{\boldsymbol{g}}}\boldsymbol{\Phi}\,\mathrm{d}x
	-\dot{\mathrm{Ad}_{\boldsymbol{g}_s}}\boldsymbol{J} \Bigr).\label{jacodot25}
\end{align}
\subsection{Jacobians of $\boldsymbol{J}\mathbf{d}$ and $\dot{\boldsymbol{J}}\mathbf{d}$}
Based on the explicit form of $\boldsymbol{J}$ and~$\dot{\boldsymbol{J}}$ in \eqref{jaco25} and \eqref{jacodot25}, given an arbitrary vector
$\mathbf{d}\in\mathbb{R}^N$, the gradients of $\boldsymbol{J}\mathbf{d}$ and $\dot{\boldsymbol{J}}\mathbf{d}$ can be deduced as follows:
\begin{equation}\label{Jq3}
	\begin{aligned}
	\frac{\partial \boldsymbol{J}}{\partial \boldsymbol{q}}\mathbf{d}
	&=\mathrm{Ad}_{\boldsymbol{g}_s}^{-1}
	\left( \int_{0}^{s} \frac{\partial\mathrm{Ad}_{\boldsymbol{g}}}{\partial\boldsymbol{q}}\boldsymbol{\Phi}\,\mathrm{d}x
	-\frac{\partial\mathrm{Ad}_{\boldsymbol{g}_s}}{\partial\boldsymbol{q}}\boldsymbol{J}\right)\mathbf{d}, \\
	\frac{\partial \dot{\boldsymbol{J}}}{\partial \boldsymbol{q}}\mathbf{d}
	&=\mathrm{Ad}_{\boldsymbol{g}_s}^{-1}
	\left( \int_{0}^{s} \frac{\partial\dot{\mathrm{Ad}_{\boldsymbol{g}}}}{\partial\boldsymbol{q}}\boldsymbol{\Phi}\, \mathrm{d}x
	-2\frac{\partial\dot{\mathrm{Ad}_{\boldsymbol{g}_s}}}{\partial\boldsymbol{q}}\boldsymbol{J}
	-\dot{\mathrm{Ad}_{\boldsymbol{g}_s}}\frac{\partial\boldsymbol{J}}{\partial\boldsymbol{q}}\right)\mathbf{d}, \\
	\frac{\partial \dot{\boldsymbol{J}}}{\partial \dot{\boldsymbol{q}}}\mathbf{d}
	&=\mathrm{Ad}_{\boldsymbol{g}_s}^{-1}
	\left( \int_{0}^{s} \frac{\partial\dot{\mathrm{Ad}_{\boldsymbol{g}}}}{\partial\dot{\boldsymbol{q}}}\boldsymbol{\Phi}\,\mathrm{d}x
	-\frac{\partial\dot{\mathrm{Ad}_{\boldsymbol{g}_s}}}{\partial\dot{\boldsymbol{q}}}\boldsymbol{J}\right)\mathbf{d}. 
\end{aligned}
\end{equation}
For $\mathbf{d}\in\mathbb{R}^6$, the gradient of $\boldsymbol{J}^\top\mathbf{d}$ becomes
\begin{equation}\label{Jq4}
	\frac{\partial \boldsymbol{J}^\top}{\partial \boldsymbol{q}}\mathbf{d}
	=\Bigl( \int_{0}^{s} \boldsymbol{\Phi}^\top
	\frac{\partial\mathrm{Ad}_{\boldsymbol{g}}^\top}{\partial\boldsymbol{q}}\,\mathrm{d}x
	-\boldsymbol{J}^\top\frac{\partial\mathrm{Ad}_{\boldsymbol{g}_s}^\top}{\partial\boldsymbol{q}} \Bigr)
	\mathrm{Ad}_{\boldsymbol{g}_s}^{-\top}\mathbf{d}.
\end{equation}
%
Equations \eqref{Jq3} and \eqref{Jq4} still contain unknown gradients of the adjoint operator $\mathrm{Ad}_{\boldsymbol{g}}$ and its time derivative.  The next step therefore isolates these terms and provides closed-form expressions for each.
\subsection{{Jacobians of $\mathrm{Ad}_{\boldsymbol{g}}\mathbf{d}$ and $\dot{\mathrm{Ad}_{\boldsymbol{g}}}\mathbf{d}$}} Given an arbitrary vector
$\mathbf{d}\in\mathbb{R}^6$, the gradients of $\mathrm{Ad}_{\boldsymbol{g}}\mathbf{d}$,  $\mathrm{Ad}^\top_{\boldsymbol{g}}\mathbf{d}$ and $\dot{\mathrm{Ad}}_{\boldsymbol{g}}\mathbf{d}$ can be deduced as follows:
\begin{enumerate}[label=(\alph*).]
	\item Using $\delta\mathrm{Ad}_{\boldsymbol{g}}=\mathrm{Ad}_{\boldsymbol{g}}\mathrm{ad}_{\delta\boldsymbol{\zeta}}$
	with $\delta\boldsymbol{\zeta}= \boldsymbol{J}\,\delta\boldsymbol{q}$, we can get
\begin{align}
			&\frac{\partial \mathrm{Ad}_{\boldsymbol{g}}}{\partial\boldsymbol{q}}\mathbf{d}
		= -\,\mathrm{Ad}_{\boldsymbol{g}}\,
		\mathrm{ad}_{\mathbf{d}}\,\boldsymbol{J},\label{Adq1}\\
		&\frac{\partial \mathrm{Ad}_{\boldsymbol{g}}^\top}{\partial\boldsymbol{q}}\mathbf{d}
		= \mathrm{ad}_{\mathrm{Ad}_{\boldsymbol{g}}^\top\mathbf{d}}^\star\,\boldsymbol{J}.\label{Adq2}
\end{align}
where $\mathrm{ad}_{(\cdot)}^\star$ is defined as a special adjoint operator of $\mathfrak{se}(3)$, defined in Appendix.
	\item Interchanging the order of time and state differentiation gives
\begin{align}
	&\frac{\partial \dot{\mathrm{Ad}_{\boldsymbol{g}}}}{\partial\boldsymbol{q}}\mathbf{d}
		= -\,\dot{\mathrm{Ad}_{\boldsymbol{g}}}\,
		\mathrm{ad}_{\mathbf{d}}\,\boldsymbol{J}
		-\mathrm{Ad}_{\boldsymbol{g}}\,
		\mathrm{ad}_{\mathbf{d}}\,\dot{\boldsymbol{J}},\label{Adq3}\\
		&\frac{\partial \dot{\mathrm{Ad}_{\boldsymbol{g}}}}{\partial \dot{\boldsymbol{q}}}\mathbf{d}
		= -\,\mathrm{Ad}_{\boldsymbol{g}}\,
		\mathrm{ad}_{\mathbf{d}}\,\boldsymbol{J}.\label{Adq4}
\end{align}
\end{enumerate}
%
Substituting the identities above into \eqref{Jq3} and \eqref{Jq4} removes the last unknowns and yields explicit gradients for~$\boldsymbol{J}\mathbf{d}$ and~$\dot{\boldsymbol{J}}\mathbf{d}$.  What remains is to assemble these building blocks into a streamlined computation strategy.
\subsection{{Algorithmic summary}}
Here, we summarize the steps for computing the Jacobian of the dynamics. The first step is to obtain each term in \eqref{sss}. This process can be divided into the following steps:
\begin{enumerate}[label=(\alph*).]
	\item Compute $\boldsymbol{J}$ and $\dot{\boldsymbol{J}}$ from \eqref{jaco25}–\eqref{jacodot25}.
	\item Substitute the auxiliary results~\eqref{Adq1}–\eqref{Adq4} into
	~\eqref{Jq3}–\eqref{Jq4} to obtain the required Jacobian gradients.
	\item Finally, use ~\eqref{eta1}–\eqref{eta2} together with those
	gradients in (b) to evaluate \eqref{sss}.
\end{enumerate}
After deriving \eqref{sss}, and considering \eqref{intertial} and \eqref{internal}, we can subsequently compute the Jacobians associated with the inertial and internal force in the dynamics as follows:
\begin{align}
&\frac{\partial \boldsymbol{\Lambda}_{ine}}{\partial\boldsymbol{q}}=\mathcal{M}\frac{\partial\dot{\boldsymbol{\eta}}}{\partial\boldsymbol{q}}-(\mathrm{ad}_{\mathcal{M}\boldsymbol{\eta}}^\star+ \mathrm{ad}^\top_{\boldsymbol{\eta}}\mathcal{M})\frac{\partial{\boldsymbol{\eta}}}{\partial\boldsymbol{q}}\label{lqq}\\
&\frac{\partial \boldsymbol{\Lambda}_{ine}}{\partial\dot{\boldsymbol{q}}}=\mathcal{M}\frac{\partial\dot{\boldsymbol{\eta}}}{\partial\dot{\boldsymbol{q}}}-(\mathrm{ad}_{\mathcal{M}\boldsymbol{\eta}}^\star+ \mathrm{ad}^\top_{\boldsymbol{\eta}}\mathcal{M})\frac{\partial{\boldsymbol{\eta}}}{\partial\dot{\boldsymbol{q}}}\label{lqq2}\\
&\frac{\partial \boldsymbol{\Lambda}_{int}}{\partial\boldsymbol{q}}=\mathcal{K}\boldsymbol{\Phi}+\mathcal{D}\frac{\partial{\boldsymbol{\eta}}}{\partial\boldsymbol{q}}+ \frac{\partial\mathcal{A}\boldsymbol{u}}{\partial\boldsymbol{q}}\label{lqq3}
\end{align}
Finally, by substituting the above expressions \eqref{lqq}-\eqref{lqq3} into the following equation \eqref{dfdq}, the Jacobian of the dynamics with respect to $\{\boldsymbol{q},\dot{\boldsymbol{q}},\ddot{\boldsymbol{q}}\}$ can be determined:
\begin{equation}\label{dfdq}
	\begin{aligned}
	&\frac{\partial\boldsymbol{f}}{\partial\boldsymbol{q}}=\int_{0}^{L}\left(\frac{\partial \boldsymbol{J}^\top}{\partial \boldsymbol{q}}\boldsymbol{\Lambda}_{ine}+\boldsymbol{J}^\top\frac{\partial \boldsymbol{\Lambda}_{ine}}{\partial\boldsymbol{q}}+ \boldsymbol{\Phi}^\top\frac{\partial \boldsymbol{\Lambda}_{int}}{\partial\boldsymbol{q}}\right) \mathrm{d}s\\
	&\frac{\partial\boldsymbol{f}}{\partial\dot{\boldsymbol{q}}}=\int_{0}^{L}\left(\boldsymbol{J}^\top\frac{\partial \boldsymbol{\Lambda}_{ine}}{\partial\dot{\boldsymbol{q}}}+ \boldsymbol{\Phi}^\top\mathcal{D}\boldsymbol{J}\right) \mathrm{d}s\\
	&\frac{\partial\boldsymbol{f}}{\partial\ddot{\boldsymbol{q}}}=\int_{0}^{L}\boldsymbol{J}^\top\mathcal{M}\boldsymbol{J}\, \mathrm{d}s
\end{aligned}
\end{equation}
Although the above steps may seem cumbersome, the matrix \( \boldsymbol{J} \), the residual of the dynamics, and its Jacobians can all be efficiently computed in a single integration pass over the interval \([0, L]\). Additionally, many intermediate terms are reused throughout the computation, allowing for substantial optimization. By strategically arranging the computation sequence, the overall computational time is not significantly increased.

After determining the residual of the dynamic equation and its Jacobians, in the following sections, we will outline the constrained optimization problem formulated for the model predictive control of the soft manipulator. The control error is defined as the objective function, while the implicit dynamic model of the soft manipulator is treated as the constraints. 
\section{NMPC-SQP Framework}\label{sec:5}
In this section, we re-formulate the time continuous optimization (\ref{optsm}) into a time discrete one as a nonlinear programming (NLP) problem, which are based on a finite-dimensional parameterization of the control trajectory. The obtained NLP is then solved by the proposed numerical optimization method.
\subsection{Definition of the Optimal Problem}
	The time interval in (\ref{optsm}) represents the prediction horizon. For a linear system with an infinite horizon, the solution can be directly obtained using the Algebraic Riccati equation \cite{willems1971least}. However, for nonlinear systems, solving for an infinite horizon is typically challenging and even computationally infeasible. Consequently, a finite prediction horizon is commonly employed. 
	In model predictive control, the prediction horizon and the dynamic model are typically discretized over the finite prediction horizon to facilitate problem formulation and solution. 
	Assuming that control inputs are constant for each sampling instant, we discretize the horizon $[t - t_{\mathrm{est}}, t + t_{\mathrm{ctrl}}]$ into the time sequence $t_{-M},t_{-M+1},\dots,t_0,t_1,\dots,t_P$, with $t_{-M}=t - t_{\mathrm{est}}$, $t_0=t$, $t_P=t + t_{\mathrm{ctrl}}$ and $t_{i+1}-t_i=h$. $h$ is the time step of discretization. Therefore, the integral in (\ref{optsm}) is transformed into a summation over each discrete time point. The estimation and prediction processes are treated separately by defining distinct time windows for each, and solving them as independent optimization problems \eqref{nlppe} and \eqref{nlpp} respectively. Denoting $\boldsymbol{x}_i=[\boldsymbol{q}_i^\top,\boldsymbol{q}_{i-1}^\top]^\top$ as the state of the system at instant $t_i$, we define the estimation problem as below: 
		\begin{equation}\label{nlppe}
		\begin{aligned}
			&\arg\min_{\boldsymbol{x}_{-M:0}} &&\sum_{i=-M}^{i=0}\left(\mathcal{J}_{\mathrm{est}}(\boldsymbol{x}_i)+  \mu_1\Vert\boldsymbol{x}_i-\boldsymbol{x}_i^-\Vert_2^2\right)\\
			&\mathrm{subject \ to} && \boldsymbol{f}(\boldsymbol{x}_{i},\boldsymbol{x}_{i-1},\boldsymbol{u}_{i})=\boldsymbol{0}\\
			&&&
			\boldsymbol{0}\leq \boldsymbol{u}_i \perp \boldsymbol{l}_i-\boldsymbol{l}_{ci}\geq \boldsymbol{0}
			\\
			&&&i=-M,-M+1,\dots,0
		\end{aligned}
	\end{equation}
	As well as the control problem:
	\begin{equation}\label{nlpp}
		\begin{aligned}
			&\arg\min_{\boldsymbol{l}_{1:P}} &&\sum_{i=1}^{i=P}\left(\mathcal{J}_{\mathrm{ctrl}}(\boldsymbol{x}_i) +  \mu_2\Vert\boldsymbol{l}_i-\boldsymbol{l}_i^-\Vert_2^2\right)\\
			&\mathrm{subject \ to} && \boldsymbol{f}(\boldsymbol{x}_{i},\boldsymbol{x}_{i-1},\boldsymbol{u}_{i})=\boldsymbol{0}\\
			&&&
				\boldsymbol{0}\leq \boldsymbol{u}_i \perp \boldsymbol{l}_i-\boldsymbol{l}_{ci}\geq \boldsymbol{0}
			\\
			&&& h\dot{\boldsymbol{l}}_{min}\leq\boldsymbol{l}_i-\boldsymbol{l}_{i-1}\leq h\dot{\boldsymbol{l}}_{max}\\
			&&&i=1,2,\dots,P
		\end{aligned}
	\end{equation}
	Here, the final constraints are designed to limit the rate of change in cable length considering the physical limits in real world. $\dot{\boldsymbol{l}}_{max}$ represents the maximum allowable cable pulling rate, while $\dot{\boldsymbol{l}}_{min}$ represents the maximum allowable cable releasing rate. 
	\begin{remark}
		To improve convergence and ensure stability, a proximal regularization term is introduced, we introduce proximal regularization term $\mu_1\Vert\boldsymbol{x}_i-\boldsymbol{x}_i^-\Vert_2^2$ and $\mu_2\Vert\boldsymbol{l}_i-\boldsymbol{l}_i^-\Vert_2^2$, where $ \mu_1 $ and $ \mu_2 $ are small positive regularization  coefficients, $\boldsymbol{x}_i^-$ and $\boldsymbol{l}_i^-$ are the reference point which are the solutions of previous iterate. 
	\end{remark}
	\subsection{Implicit time differentiation}
	Since the optimization problem has been discretized in time, it is necessary to convert the continuous-time dynamic constraints into their discrete-time counterparts.
	Implicit time-stepping is a widely used method for the time discretization of dynamic systems and has become increasingly popular in robotics for both simulation and control applications \cite{shabana2020dynamics}.
	Let's consider a small time interval $[t, t_{-}]$, and let $h=t-t_{-}$. Denote $\boldsymbol{x}=[\boldsymbol{q}^\top,\boldsymbol{q}_{-}^\top]^\top$ as the state at the current time point $t$, and $\boldsymbol{x}_-=[\boldsymbol{q}_-^\top,\boldsymbol{q}_{--}^\top]^\top$ as these at the last time point $t_{-}$. Applying the implicit Euler method, we approximate the first and second time derivatives as
 \begin{equation}\label{imp}
 	{\boldsymbol{q}}=\boldsymbol{\alpha}\boldsymbol{x}, \quad \dot{\boldsymbol{q}}=\boldsymbol{\beta}\boldsymbol{x}, \quad \ddot{\boldsymbol{q}}=\boldsymbol{\gamma}(\boldsymbol{x}-\boldsymbol{x}_-)
 \end{equation}
	where $\boldsymbol{\alpha}=[\mathbf{I}_N, \ \mathbf{0}_N]$, $\boldsymbol{\beta}=[\mathbf{I}_N, \ -\mathbf{I}_N]/h$, $\boldsymbol{\gamma}=[\mathbf{I}_N, \ -\mathbf{I}_N]/h^2$.
	Substituting these approximations into the continuous-time dynamics given in \eqref{dy23}, we obtain the corresponding time-discrete dynamics:
	\begin{equation}\label{dy23d}
		\begin{aligned}
			\boldsymbol{f}({\boldsymbol{x}},\boldsymbol{x}_-,\boldsymbol{u})=\int_{0}^{L}&\boldsymbol{J}^\top(\mathcal{M}\dot{\boldsymbol{\eta}}-\mathrm{ad}^\top_{\boldsymbol{\eta}}\mathcal{M}\boldsymbol{\eta}+\boldsymbol{\Lambda}_{ext})\\&+ \boldsymbol{\Phi}^\top(\mathcal{K}\boldsymbol{\Phi}\boldsymbol{q}+\mathcal{D}\boldsymbol{\eta}+ \mathcal{A}\boldsymbol{u} )\ \mathrm{d}s
		\end{aligned}
	\end{equation}
	where $\boldsymbol{\eta}=\boldsymbol{J}\boldsymbol{\beta}\boldsymbol{x}, \ \dot{\boldsymbol{\eta}}=\dot{\boldsymbol{J}}\boldsymbol{\beta}\boldsymbol{x}+
	\boldsymbol{J}\boldsymbol{\gamma}(\boldsymbol{x}-\boldsymbol{x}_-)$. The Jacobian of the time-discrete dynamics with respect to the states $\boldsymbol{x}$ and $\boldsymbol{x}_-$ can be straightforwardly obtained by applying the chain rule of differentiation based on \eqref{dfdq} and \eqref{imp}:
	\begin{align}
		\frac{\partial\boldsymbol{f}}{\partial\boldsymbol{x}}=\frac{\partial\boldsymbol{f}}{\partial\boldsymbol{q}}\boldsymbol{\alpha}+\frac{\partial\boldsymbol{f}}{\partial\dot{\boldsymbol{q}}}\boldsymbol{\beta}+\frac{\partial\boldsymbol{f}}{\partial\ddot{\boldsymbol{q}}}\boldsymbol{\gamma}, \quad
		\frac{\partial\boldsymbol{f}}{\partial\boldsymbol{x}_-}=-\frac{\partial\boldsymbol{f}}{\partial\ddot{\boldsymbol{q}}}\boldsymbol{\gamma}
	\end{align}
\subsection{Smooth Cable Driven Constraint}\label{sec:rsd}
Both \eqref{nlppe} and  \eqref{nlpp} contain the nonlinear complementarity constraint. Traditionally, the solution of nonlinear complementarity problem is first approximated by a linear one and subsequently solved using  active-set methods \cite{dirkse1995path}, which meticulously ensure complementarity at every iteration.
However, the quality of these computations heavily depends on the optimization method used. Pivoting methods, which enforce strict complementarity at each iteration, often return solutions at non-differentiable points. As a result, this differentiation yields subgradients, which can make typically efficient second-order optimization slower and less reliable. Therefore, we propose the following alternative smoothing method to approximate the complementarity conditions while ensuring gradient smoothness.

To address this challenge, we employed the nonlinear complementarity function [21]. Specifically, we redefined the tension $T$ and length difference $l-l_{c}$ using the softplus function $E(\lambda)$ that incorporates a slack variable $\lambda$:
\begin{equation}\label{zz}
	T = E(\lambda) \ , \ l-l_{c} = E(-\lambda)
\end{equation}
where the softplus function is defined as below:
$$E(\lambda)=\frac{1}{c}\log (1+e^{c\lambda}) \ , \ c>0$$
The parameter $c$ plays a crucial role in determining the level of smoothing in (\ref{zz}). As $c$ approaches infinity, (\ref{zz}) progressively converges to the plus function $\max(0,\lambda)$, therefore ensures the complementary constraint for $T$ and $l-l_{c}$.
We define the following variables for more compact notation considering $ m $ cables: $$\boldsymbol{\lambda}=[\lambda_1,\lambda_2,\dots,\lambda_m]\in\mathbb{R}^{m}$$  $$\boldsymbol{E}(\boldsymbol{\lambda})=[E(\lambda_1),E(\lambda_2),\dots,E(\lambda_m)]\in\mathbb{R}^{m}$$
In \eqref{nlppe} and \eqref{nlpp}, we substitute the cable tension vector $\boldsymbol{\tau}$ as a function of the slack variable, i.e., $\boldsymbol{u}=\boldsymbol{E}(\boldsymbol{\lambda})$. Consequently, the control input to the dynamic system is reformulated in terms of $\boldsymbol{\lambda}$: 
\begin{equation}\label{tr1}
	\boldsymbol{f}(\boldsymbol{x}_{i},\boldsymbol{x}_{i-1},\boldsymbol{u}_{i})=\boldsymbol{0} \Rightarrow \boldsymbol{f}(\boldsymbol{x}_{i},\boldsymbol{x}_{i-1},\boldsymbol{\lambda}_{i})=\boldsymbol{0} 
\end{equation}
Subsequently, the \xun{NCP} in \eqref{nlppe} and \eqref{nlpp} are replaced by the equality constraints:
\begin{equation}\label{tr2}
	\boldsymbol{0}\leq \boldsymbol{u}_i \perp \boldsymbol{l}_i-\boldsymbol{l}_{ci}\geq \boldsymbol{0} \Rightarrow\boldsymbol{G}(\boldsymbol{x},\boldsymbol{l},\boldsymbol{\lambda})=\boldsymbol{0}
\end{equation}
where
$\boldsymbol{G}(\boldsymbol{x},\boldsymbol{l},\boldsymbol{\lambda})$ is defined as $ \boldsymbol{l}-\boldsymbol{l}_c-\boldsymbol{E}(-\boldsymbol{\lambda})$. 
Subsequently, $\boldsymbol{l}$ can also be defined explicitly by $\boldsymbol{\lambda}$ and $\boldsymbol{x}$: 
\begin{equation}\label{clll}
	\boldsymbol{l} = \boldsymbol{l}_c+\boldsymbol{E}(-\boldsymbol{\lambda})
\end{equation}
\subsection{Redefinition of the Optimal Problem}
The basic control idea is to use the current state $\boldsymbol{x}_0$ and the target trajectory to solve for the cable lengths $\boldsymbol{l}_{1:P}$ at each time point within the prediction horizon, according to the aforementioned optimization problem. The cable length at the first time point is then used as the control input for the soft manipulator.
In general, the optimal control problem is re-evaluated after the  sampling time step $h$. Using the new system state at time $t+h$, the whole procedure (i.e., prediction and optimization) is repeated, moving the control forward \cite{allgower2004nonlinear}. Ensuring the real-time solvability of the optimal control problem is crucial. Due to the multi-dimensional and multi-constraint nature of the associated optimization problem, achieving a fast solution is often challenging.

In general, solving an NLP such as (\ref{nlpp}) can be categorized into two main types: sequential approach and simultaneous approach. The basic idea of the sequential approach is to use the dynamic model to move forward from the current state, sequentially calculating the system states at each discrete time point within the prediction horizon and estimating the error with respect to the target trajectory. However, for the dynamics of soft manipulator, because the system states within the prediction window depend very nonlinearly on the initial state of the prediction horizon, obtaining the gradient of the objective function with respect to the system inputs (i.e., cable lengths) is very challenging and complex. Nonlinearity accumulates gradually from the initial time point within the prediction horizon, leading to a significant increase in the sensitivity of the soft manipulator’s deformation to cable length inputs at later time points.

In contrast, simultaneous approach involves parameterizing the state trajectory as optimization variables within the nonlinear programming (NLP) problem, while imposing suitable equality constraints that model the dynamics \cite{diehl2006fast}. This approach allows for the simultaneous execution of simulation and optimization. The state trajectory only accurately reflects a valid solution in relation to the control trajectory once the NLP has been solved. 
This contrasts with single shooting, where the nonlinearity of the system does not build up over the entire horizon. In the following sections, we will begin by parameterizing the input trajectory and focus on the solution methods for simultaneous approach.

Within each prediction horizon, considering the transform \eqref{tr1} as well as \eqref{tr2}, we redefine the initial estimation problem (\ref{nlppe}) as:
	\begin{equation}\label{nlppe2}
	\begin{aligned}
		&\arg\min_{(\boldsymbol{x},\boldsymbol{\lambda})_{-M:0}} &&\sum_{i=-M}^{i=0}\left(\mathcal{J}_{\mathrm{est}}(\boldsymbol{x}_i)+  \mu_1\Vert\boldsymbol{x}_i-\boldsymbol{x}_i^-\Vert_2^2\right)\\
		&\mathrm{subject \ to} && \boldsymbol{f}(\boldsymbol{x}_{i},\boldsymbol{x}_{i-1},\boldsymbol{\lambda}_{i})=\boldsymbol{0}\\
		&&&
\boldsymbol{G}(\boldsymbol{x}_i,\boldsymbol{l}_i,\boldsymbol{\lambda}_i)=\boldsymbol{0}
		\\
		&&&i=-M,-M+1,\dots,0
	\end{aligned}
\end{equation}
and redefine the control problem (\ref{nlpp}) as: 
\begin{equation}\label{nlpp2}
	\begin{aligned}
		&\arg\min_{(\boldsymbol{x},\boldsymbol{\lambda})_{1:P}} &&\sum_{i=1}^{i=P}\left(\mathcal{J}_{\mathrm{ctrl}}(\boldsymbol{x}_i)+  \mu_2\Vert\boldsymbol{\lambda}_i-\boldsymbol{\lambda}_i^-\Vert_2^2\right) \\
		&\mathrm{subject \ to} &&\boldsymbol{f}(\boldsymbol{x}_{i},\boldsymbol{x}_{i-1},\boldsymbol{\lambda}_{i})=\boldsymbol{0}\\
		&&&h\dot{\boldsymbol{l}}_{max}-\boldsymbol{l}_i+\boldsymbol{l}_{i-1}\geq \boldsymbol{0}\\
		&&&\boldsymbol{l}_i-\boldsymbol{l}_{i-1}-h\dot{\boldsymbol{l}}_{min}\geq\boldsymbol{0}\\
		&&&i=1,2,\dots,P
	\end{aligned}
\end{equation}
Note that in \eqref{nlpp2}, we do not explicitly define \eqref{tr2} in the constraints, but instead directly use \eqref{clll} to compute
the cable length $\boldsymbol{l}$.
Given the cable length sequence \(\{\boldsymbol{l}_{-M}, \dots, \boldsymbol{l}_0\}\) over the past estimation horizon \([t - t_{\text{est}}]\), the corresponding system states and slack variables, i.e., \(\{\boldsymbol{x}_{-M}, \dots, \boldsymbol{x}_0, \boldsymbol{\lambda}_{-M}, \dots, \boldsymbol{\lambda}_0\}\), can be estimated by solving the optimization problem \eqref{nlppe2}. \xun{In fact, the states $\boldsymbol{x}_{-M}$ to $\boldsymbol{x}_{-1}$ correspond to $x_{-M+1}$ to $x_{0}$ from the previous estimation horizon and are therefore known in the current estimation horizon. The same holds for $\boldsymbol{\lambda}$. Consequently, for the current estimation horizon, the only unknowns are $x_{0}$ and $\lambda_{0}$.}

Once the current state \(\boldsymbol{x}_0\) is obtained from \eqref{nlppe2}, the optimization problem \eqref{nlpp2} can then be solved to determine the future trajectories \(\{\boldsymbol{x}_1, \dots, \boldsymbol{x}_P, \boldsymbol{\lambda}_1, \dots, \boldsymbol{\lambda}_P\}\), which fully characterize the optimal state and input over the prediction horizon \([t + t_{\text{ctrl}}]\).

Due to the highly nonlinear nature of (\ref{nlppe2}) and (\ref{nlpp2}), finding a solution will be challenging. The following subsection will outline our proposed methods used to address this issue.
\subsection{Solving method}\label{sqpp}
%
In this work, we adopt sequential quadratic programming (SQP) to solve the nonlinear optimization problems arising in NMPC. SQP offers rapid local convergence, typically quadratic near optimal solutions, ensuring high accuracy within limited iterations. Its iterative structure naturally aligns with the NMPC framework, allowing efficient warm-starting from previous solutions, which significantly reduces computational time. Additionally, SQP effectively handles nonlinear and coupled constraints, robustly ensuring feasibility of system states and inputs. 

Specifically, we convert the NLP \eqref{nlpp2} into a series of quadratic sub-problems (QP).
To proceed, we first introduce the Lagrangian function for \eqref{nlpp2}. Let $\boldsymbol{z}$ represent the vector of decision variables, i.e., the sequence of $\boldsymbol{x}$ and $\boldsymbol{\lambda}$ at discrete time points in the time horizon. The objective function in \eqref{nlpp2} is denoted by ${a}(\boldsymbol{z})$. We encapsulate all the equality constraints as $\boldsymbol{b}(\boldsymbol{z}) = \boldsymbol{0}$ and all the inequality constraints as $\boldsymbol{c}(\boldsymbol{z}) \geq \boldsymbol{0}$. Thus, the Lagrangian function for (\ref{nlpp2}) is given by:
$$\mathcal{L}(\boldsymbol{z},\boldsymbol{\gamma},\boldsymbol{\zeta})={a}(\boldsymbol{z})-\boldsymbol{\gamma}^\top\boldsymbol{b}(\boldsymbol{z})-\boldsymbol{\zeta}^\top\boldsymbol{c}(\boldsymbol{z})$$
with the Lagrange multipliers $\boldsymbol{\gamma}$ and $\boldsymbol{\zeta}$, which are essential in the optimization process. For a point \(\boldsymbol{z}\) to be a local optimum of the NLP (\ref{nlpp2}), the necessary conditions are that there exist multipliers \(\boldsymbol{\gamma}\) and \(\boldsymbol{\zeta}\) such that
$$\nabla_{\boldsymbol{z}}\mathcal{L}(\boldsymbol{z},\boldsymbol{\gamma},\boldsymbol{\zeta})=\boldsymbol{0},\ \boldsymbol{b}(\boldsymbol{z})=\boldsymbol{0},\ \boldsymbol{0}\leq \boldsymbol{c}(\boldsymbol{z})\perp \boldsymbol{\zeta}\geq \boldsymbol{0}$$
We approach the solution $(\boldsymbol{z},\boldsymbol{\gamma},\boldsymbol{\zeta})$ iteratively. Beginning with an initial guess $(\boldsymbol{z}_0,\boldsymbol{\gamma}_0,\boldsymbol{\zeta}_0)$, the iteration is given by:
$$\boldsymbol{z}_{k+1}=\boldsymbol{z}_k+\Delta\boldsymbol{z}_k,\ \boldsymbol{\gamma}_{k+1}=\boldsymbol{\gamma}_k^{QP} , \ \boldsymbol{\zeta}_{k+1}=\boldsymbol{\zeta}_k^{QP}$$
where $(\Delta\boldsymbol{z}_k,\boldsymbol{\gamma}_k^{QP},\boldsymbol{\zeta}_k^{QP})$ represents the solution to a sub QP as follows:
\begin{equation}
\begin{aligned}
&\arg\min\limits_{\Delta\boldsymbol{z}_k}
&&\frac{1}{2}\Delta\boldsymbol{z}_k^{\top}\mathcal{H}_k\Delta\boldsymbol{z}_k+\mathcal{G}_k^\top\Delta\boldsymbol{z}_k\\
&\mathrm{subject \ to} &&
\boldsymbol{b}(\boldsymbol{z}_k)+\nabla_{\boldsymbol{z}}^\top\boldsymbol{b}(\boldsymbol{z}_k)\Delta\boldsymbol{z}_k=\mathbf{0}\\
&&&\boldsymbol{c}(\boldsymbol{z}_k)+\nabla_{\boldsymbol{z}}^\top\boldsymbol{c}(\boldsymbol{z}_k)\Delta\boldsymbol{z}_k\geq\mathbf{0}
\end{aligned}
\label{eq:SQP}
\end{equation}
\begin{figure*}[t]
	\centering
	\includegraphics[width=0.9\textwidth]{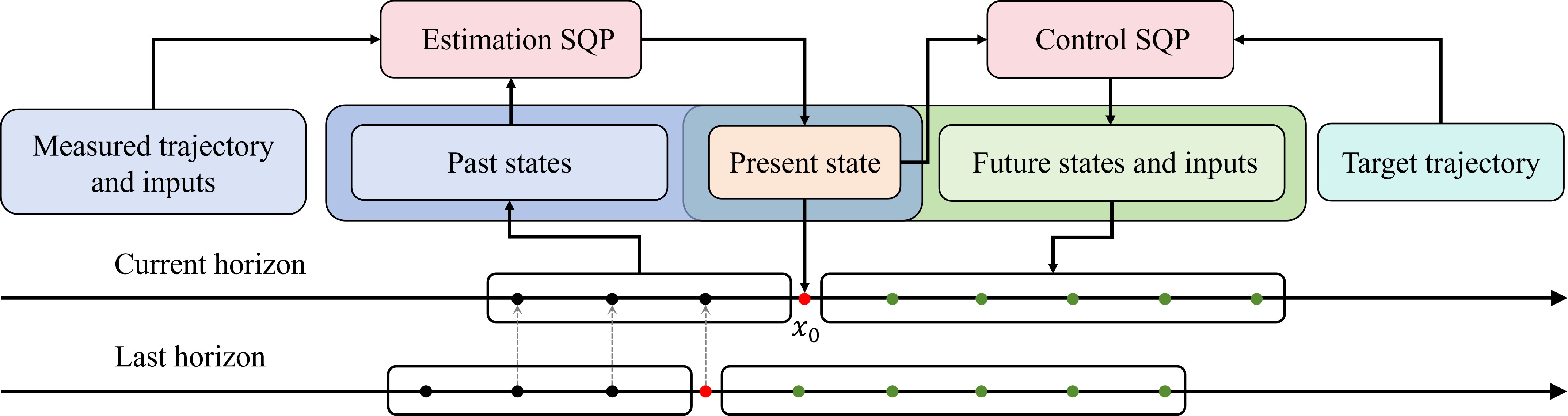}
	\caption{Sequential Quadratic Programming based framework for coupled estimation and control.}
	\label{framework}
\end{figure*}

Here,
$\mathcal{G}_k^\top$ is the gradient of objective function ${a}(\boldsymbol{z}_k)$ about $\boldsymbol{z}_k$. For brevity, we do not explicitly expand the gradient here. However, it can be readily computed using the gradients  \eqref{jacog} derived in the previous Section \ref{secposer}.
$\mathcal{H}_k$ is the Hessian of the Lagrangian, i.e.,
$
	\mathcal{H}_k=\nabla^2\mathcal{L}(\boldsymbol{z}_k,\boldsymbol{\gamma}_k,\boldsymbol{\zeta}_k)
$, which we approximated by Gauss-Newton Hessian \cite{nocedal1999numerical}.
$\nabla_{\boldsymbol{z}}^\top\boldsymbol{b}(\boldsymbol{z}_k)$ and $\nabla_{\boldsymbol{z}}^\top\boldsymbol{c}(\boldsymbol{z}_k)$ are the Jacobians of equality constraints and inequality constraints. 

According to \eqref{eq:SQP}, we can deduce the sub QP of the estimation problem \eqref{nlppe2} as 
	\begin{equation}\label{nlppe3}
	\begin{aligned}
		&\arg\min_{\Delta(\boldsymbol{x},\boldsymbol{\lambda})_{0}} &&\sum_{i=-M}^{i=0}\left(\frac{1}{2}\Delta\boldsymbol{x}_i^{\top}\boldsymbol{H}_i\Delta\boldsymbol{x}_i+\nabla^\top_{\boldsymbol{x}_i}\mathcal{J}_i\Delta\boldsymbol{x}_i+c_i\right)\\
		&\mathrm{subject \ to} && \boldsymbol{f}_i+\boldsymbol{A}_i\Delta\boldsymbol{x}_i+\boldsymbol{B}_i\Delta\boldsymbol{x}_{i-1}+\boldsymbol{C}_i\Delta\boldsymbol{\lambda}_i=\boldsymbol{0}\\
		&&&
		\boldsymbol{G}_i+\nabla_{\boldsymbol{x}_i}\Delta\boldsymbol{x}_i+\nabla_{\boldsymbol{\lambda}_i}\Delta\boldsymbol{\lambda}_i=\boldsymbol{0}
		\\
		&&&i=-M,-M+1,\dots,0
	\end{aligned}
\end{equation}
and deduce the sub QP of control problem (\ref{nlpp2}) as: 
\begin{equation}\label{nlpp3}
	\begin{aligned}
		&\arg\min_{\Delta(\boldsymbol{x},\boldsymbol{\lambda})_{1:P}} &&\sum_{i=1}^{i=P}\left(\frac{1}{2}\Delta\boldsymbol{x}_i^{\top}\boldsymbol{H}_i\Delta\boldsymbol{x}_i+\nabla^\top_{\boldsymbol{x}_i}\mathcal{J}_i\Delta\boldsymbol{x}_i+p_i\right) \\
		&\mathrm{subject \ to}&&
		\boldsymbol{f}_i+\boldsymbol{A}_i\Delta\boldsymbol{x}_i+\boldsymbol{B}_i\Delta\boldsymbol{x}_{i-1}+\boldsymbol{C}_i\Delta\boldsymbol{\lambda}_i=\boldsymbol{0}\\
		&&&h\dot{\boldsymbol{l}}_{max}-\boldsymbol{l}_i+\boldsymbol{l}_{i-1}-\Delta\boldsymbol{l}_i+\Delta\boldsymbol{l}_{i-1}\geq \boldsymbol{0}\\
		&&&\boldsymbol{l}_i-\boldsymbol{l}_{i-1}-h\dot{\boldsymbol{l}}_{min}+\Delta\boldsymbol{l}_i-\Delta\boldsymbol{l}_{i-1}\geq\boldsymbol{0}\\
		&&&i=1,2,\dots,P
	\end{aligned}
\end{equation}
Here, \(\boldsymbol{H}_i\) denotes the Hessian matrix of \(\mathcal{J}_i\) with respect to \(\boldsymbol{x}_i\). The terms \(c_i\) and \(p_i\) represent the linearized components of the proximal regularization term with respect to \(\boldsymbol{x}_i\) and \(\boldsymbol{\lambda}_i\), respectively, as given by:
$$c_i=2\mu_1(\boldsymbol{x}_i-\boldsymbol{x}_i^-)^\top\Delta\boldsymbol{x}_i+\mu_1\Delta\boldsymbol{x}_i^\top\Delta\boldsymbol{x}_i,$$
$$p_i=2\mu_2(\boldsymbol{\lambda}_i-\boldsymbol{\lambda}_i^-)^\top\Delta\boldsymbol{\lambda}_i+\mu_2\Delta\boldsymbol{\lambda}_i^\top\Delta\boldsymbol{\lambda}_i.$$
In the constraints, for notational convenience, we define $\boldsymbol{f}_i=\boldsymbol{f}(\boldsymbol{x}_{i},\boldsymbol{x}_{i-1},\boldsymbol{\lambda}_{i})$, $\boldsymbol{A}_i=\nabla^\top_{\boldsymbol{x}_i}\boldsymbol{f}_i$, $\boldsymbol{B}_i=\nabla^\top_{\boldsymbol{x}_{i-1}}\boldsymbol{f}_i$ and $\boldsymbol{C}_i=\nabla^\top_{\boldsymbol{\lambda}_i}\boldsymbol{f}_i$. The analytical expressions for all these terms have been presented or derived in the preceding sections.

After a search direction $\Delta(\boldsymbol{x},\boldsymbol{\lambda})$ has been calculated, an integral step size $r$ is determined in order to obtain the next iteration
$$\boldsymbol{x}_{k+1} =\boldsymbol{x}_k+r\Delta\boldsymbol{x}_k
$$
$$\boldsymbol{\lambda}_{k+1} =\boldsymbol{\lambda}_k+r\Delta\boldsymbol{\lambda}_k
$$
\begin{remark}
	Here the step size is tried until some acceptance criterion is satisfied. By convention, a trial step size $r\in(0,1]$ is accepted if the corresponding trial point provides sufficient reduction of a merit function \cite{wachter2005line}. Additionally, to ensure real-time computation, we imposed an upper limit on the number of iterations for the sub-QP.
\end{remark}
\begin{algorithm}[t]\label{Algo:QP2}
	\caption{Pseudo-code of algorithm implementation}
	\begin{algorithmic}[1]
		\Require
		$\widehat{\boldsymbol{x}}_{-M:0}$: initial estimated state sequence; $\boldsymbol{x}_{0:P}$: initial predicted state sequence;
		$\boldsymbol{\theta}_0$: initial control input parameter;
		\Ensure
		estimated state, optimal state and control input;
		\State Set horizon window of MHE as $[t-Mh,t]$. Set horizon window of MPC as $[t,t+Ph]$. Choose the interpolation matrix $\boldsymbol{\Phi}_l$. Set the final time as ${T_{max}}$;
		\While{$t\leq{T_{max}}$}
		\State $t_0\gets t$
		\For{$k=1:\mathrm{num\_QP\_estimation}$}
		\For{$i=-M:0$}
		\State $t_i \gets t_0+ih$ 
		\State Measure $\boldsymbol{g}(t_i)$ and $\boldsymbol{l}(t_i)$
		\State  Compute  $\mathcal{J}_{\mathrm{est}}(\widehat{\boldsymbol{x}}_i)$ \hfill $\triangleright$ \eqref{jest}
		\State  Compute $\boldsymbol{f}(\widehat{\boldsymbol{x}}_{{i}},\widehat{\boldsymbol{x}}_{{i-1}},\boldsymbol{\lambda}_{i})$ \hfill $\triangleright$ \eqref{tr1}
		\State Compute $\boldsymbol{G}(\widehat{\boldsymbol{x}}_i,\boldsymbol{l}_i,\widehat{\boldsymbol{\lambda}}_i)$\hfill $\triangleright$ \eqref{tr2}
		\EndFor
		\State Construct and solve $\mathrm{QP\_estimation}$\hfill $\triangleright$ \eqref{nlppe3}
		\State Update $(\widehat{\boldsymbol{x}},\widehat{\boldsymbol{\lambda}})_{-M:0}$
		\EndFor		
		\State $\boldsymbol{x}_0\gets\widehat{\boldsymbol{x}}_0$
		\For{$k=1:\mathrm{num\_QP\_control}$}
		\For{$i=1:P$}
		\State $t_i \gets t_0+ih$ 
		\State Get $\boldsymbol{g}_r(t_i)$	
		\State  Compute  $\mathcal{J}_{\mathrm{ctrl}}({\boldsymbol{x}}_i)$\hfill $\triangleright$ \eqref{jctrl}
		\State Compute $\boldsymbol{l}_i$\hfill $\triangleright$ \eqref{clll}
		\State  Compute $\boldsymbol{f}(\boldsymbol{x}_{{i}},\boldsymbol{x}_{{i-1}},\boldsymbol{\lambda}_i)$ \hfill $\triangleright$ \eqref{tr1}
		\EndFor
		\State Construct and solve $\mathrm{QP\_control}$\hfill $\triangleright$ \eqref{nlpp3}
		\State Update $({\boldsymbol{x}},{\boldsymbol{\lambda}})_{1:P}$
		\EndFor
		\State Compute $\boldsymbol{l}_1$\hfill $\triangleright$ \eqref{clll}
		\State Soft manipulator $\gets\boldsymbol{l}_1$
		\State $t\gets t+h$
		\EndWhile
	\end{algorithmic}
\end{algorithm}

In summary, when performing state estimation and control for the soft manipulator, the nonlinear optimization problems associated with estimation and control are linearized into a series of sub-QPs within each time horizon, which are then solved sequentially. The overall procedure of using SQP for state estimation and control of the soft manipulator is summarized in \Cref{framework} and Algorithm~1.
\section{Numerical Simulations}\label{sec:7}
In this section, the simulation environment for a soft manipulator was developed within the MATLAB framework. The simulation loop encompasses dynamic computations of the soft manipulator, integrating closed-loop control mechanisms utilizing MHE and NMPC. The dynamics of the soft manipulator are resolved using implicit Euler integration. The computational setup includes the Dell Precision 7680 laptop, equipped with a 13th Gen Intel(R) Core(TM) i7-13850HX 2.10 GHz CPU.
\begin{figure}[h]
	\centering
	\includegraphics[width=0.49\textwidth]{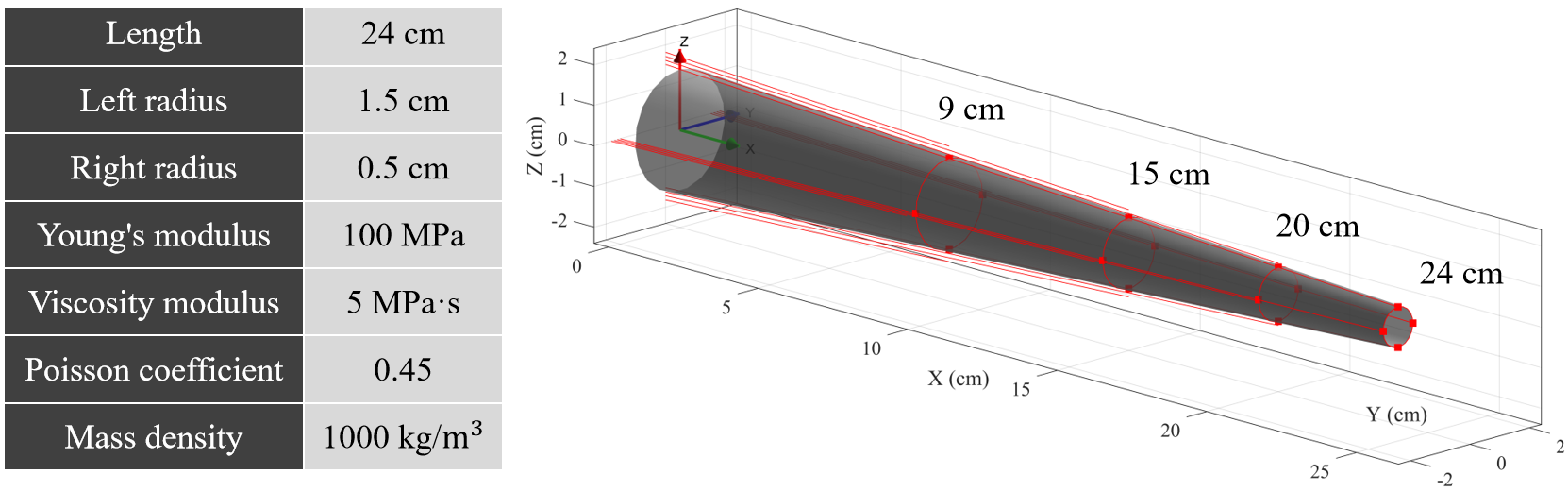}
	\caption{Simulated soft manipulator used for numerical validation.}
	\label{simurobot2}
\end{figure}

As shown in Fig. \ref{simurobot2}, the trunk of the soft manipulator is a conical, homogeneous soft rod, with its centerline aligned along the x-axis of the global frame. The left end of the soft manipulator is fixed to the base, while the right end is free.
The soft manipulator is driven by 16 cables. 
Each cable is embedded in the surface of the soft manipulator's trunk and is attached at its end to the trunk. The cables are divided into four groups, each containing four cables, with the ends fixed to cross-sections of the soft manipulator located at axial distances of 9 cm, 15 cm, 20 cm, and 24 cm from the base.

This section aims to validate the feasibility of the MHE-NMPC framework through simulation and analyze the simulation results. We will subsequently introduce the simulation of the observer based on MHE, as well as the closed-loop control simulation based on MHE-NMPC.
\subsection{Estimation of configuration via MHE}
\subsubsection{Scenario definition}
In this subsection, we will validate the feasibility and estimation accuracy of the MHE algorithm through simulation. To demonstrate the effectiveness of the MHE algorithm, we have made some assumptions in the simulation based on general real-world situation.

During the movement of the soft manipulator, we assume that its strain and trunk pose cannot be directly measured. A pose sensor is installed at the end of the soft manipulator, i.e., its end-effector, which can measure the position and Euler angles of the end-effector in the global frame. Meanwhile, the tension in each cable is not measurable, but the length of each cable's pull is measurable. Therefore, during the movement of the manipulator, the measurable physical quantities are the pose of the end-effector and the pull lengths of each cable. These physical quantities serve as inputs to the MHE algorithm, allowing it to estimate the strain distribution and the trunk pose of the soft manipulator.

As shown in Fig. \ref{conf_MHE}, the initial state ($t=0s$) of the soft manipulator is in a steady state under gravity. The initial estimated values for the configuration of the soft manipulator are set to its natural, horizontal state without any applied forces, with all strains set to zero, i.e., $ \boldsymbol{q}_{obs} = \boldsymbol{0} $. We assume that at this moment, the soft manipulator is driven by the cable tensions and begins to deform. The variations in the tensions of all the cables are shown in Fig. \ref{fig:MHE_T}.

Here, we use the piecewise linear strain method to divide the soft manipulator into four sections, considering three bending strains along the x, y, and z directions expressed in body frame, and one stretching strain along the x direction expressed in body frame. Therefore, the continuous strain field of the soft manipulator is discretized into a vector $ \boldsymbol{q} \in \mathbb{R}^{20}$. Similarly, the strain field estimated by the observer is represented by a vector $ \boldsymbol{q}_{obs}\in \mathbb{R}^{20} $.
\begin{figure}[t]
	\centering
	\includegraphics[width=0.49\textwidth]{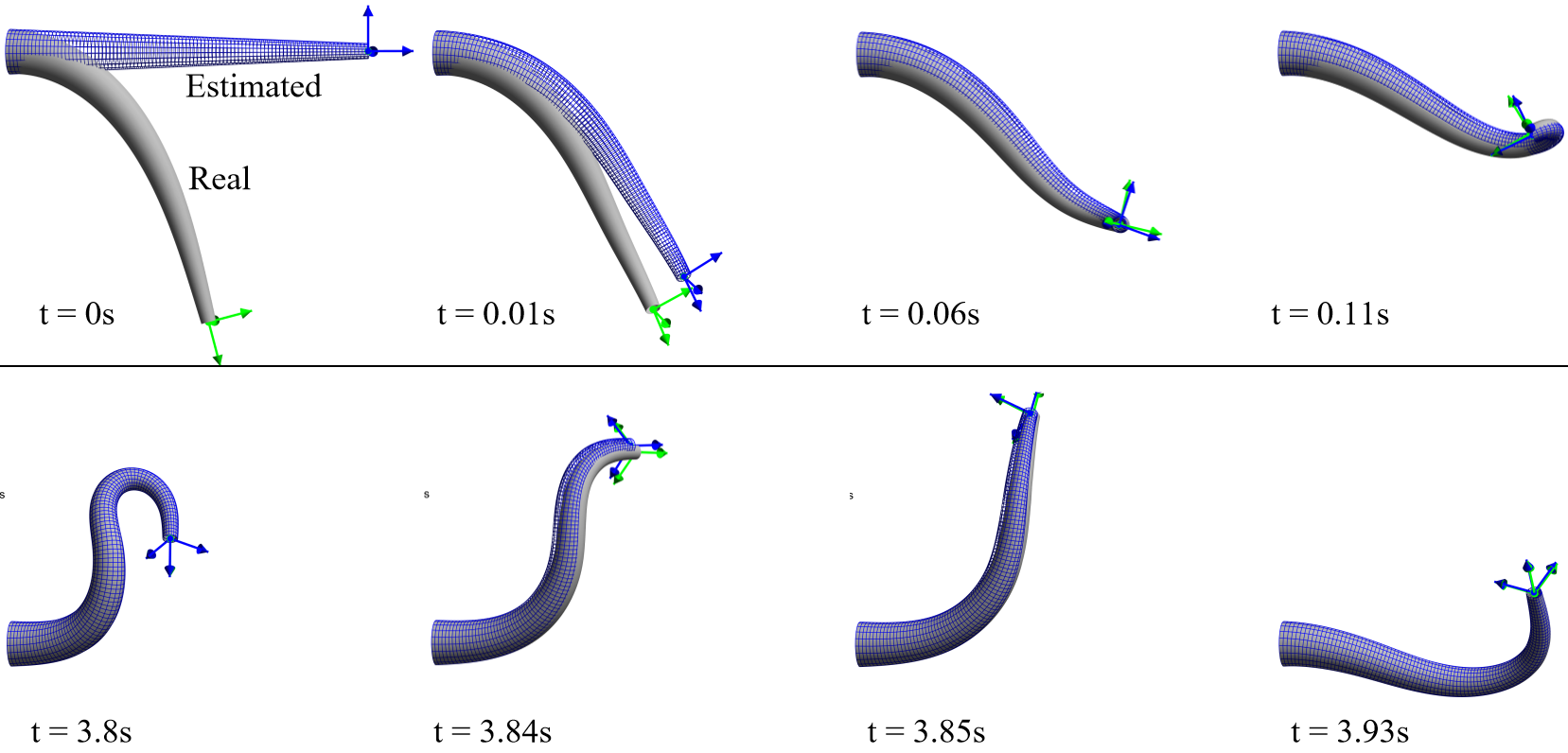}
	\caption{The real configuration (depicted by the gray body) and the estimated configuration (depicted by the blue mesh body) of the soft manipulator at different time instances are illustrated. The simulation time step is set to $\mathrm{dt} = 0.01\mathrm{s}$.}
	\label{conf_MHE}
\end{figure}
\begin{figure}[t]
	\centering
	\subfigure[The tension of each cable ]{\includegraphics[width=0.49\textwidth]{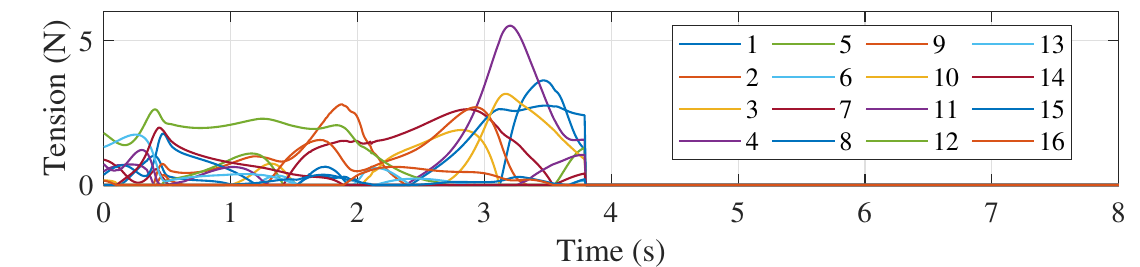}\label{fig:MHE_T}}
	\\
	\subfigure[The error of strain estimation ]{\includegraphics[width=0.49\textwidth]{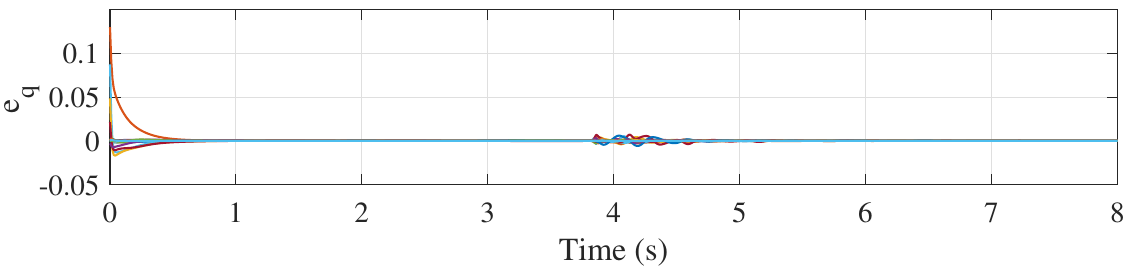}\label{fig:MHE_error}}
	\caption{Variation of cable tensions and strain error over time}
	\label{fig:MHE}
\end{figure}
\subsubsection{Analysis of simulation result}
During the deformation process of the soft manipulator, its pose estimates gradually converge to the true pose, as shown from $0$ seconds to $0.11$s in Fig. \ref{conf_MHE}. To verify the robustness of the observer, the cable drive is stopped at $3.8$s. At this moment, due to the sudden disappearance of cable tension, the soft manipulator's configuration undergoes a significant abrupt change. The simulation results indicate that the observer can still quickly track the true pose of the soft manipulator, as demonstrated from $3.8$s to $3.93$s in Fig. \ref{conf_MHE}.

Fig. \ref{fig:MHE_error} shows the tracking error of the strain with $\boldsymbol{e}_q=\boldsymbol{q}-\boldsymbol{q}_{obs}$. The results indicate that the observer has rapid convergence. Within $0.5$s, the strain error converges to within $\pm 0.001$. When the disturbance occurs at $3.8$s, the strain error remains within $\pm 0.01$ and reconverges to within $\pm 0.001$ within $1$s.

Based on the configuration and strain of the soft manipulator estimated by the MHE observer, we can now input the observed strain into the MPC controller. This will allow us to conduct the control simulation tests detailed in the subsequent sections.
\subsection{Constant pose control}
The first control test aims to control the end-effector to rapidly move to a fixed target pose, encompassing both position and orientation.
\begin{figure}[h]
	\centering
	\includegraphics[width=0.42\textwidth]{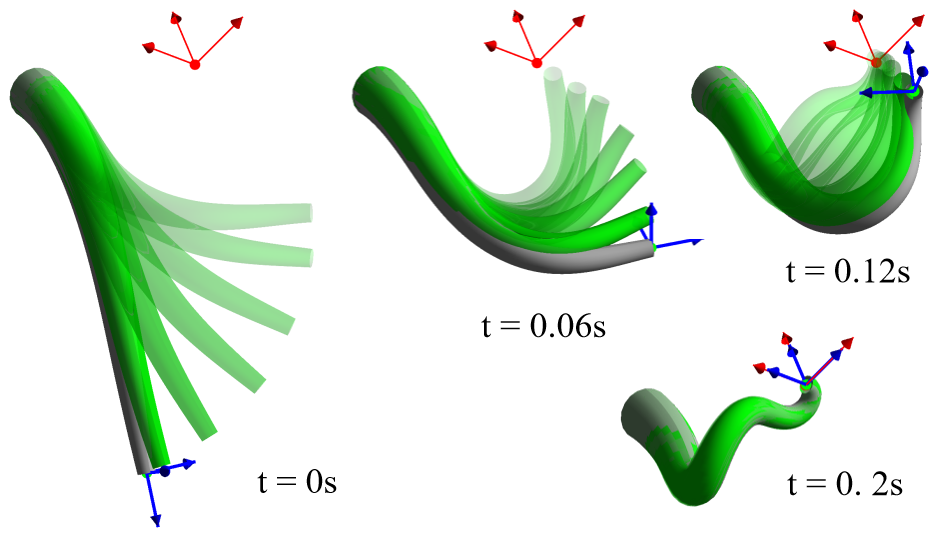}
	\caption{The configurations of the soft robot (gray rod) at various time instances during the control process, along with their corresponding predicted configurations (green rod) within the prediction horizon, are illustrated. The figure shows the predicted configurations at each discretized time point within the prediction horizon, with a time interval of $0.01$s  between each pair of adjacent predicted configurations. The simulation time step is set to $0.005$s.}
	\label{constant}
\end{figure} 
\subsubsection{Control objective}\label{B1}
As shown in Fig. \ref{constant}, at the initial moment, the soft manipulator is in a steady state under the influence of gravity. In this test, we set the target pose of the end-effector as follows:
$$\boldsymbol{g}_r=\begin{bmatrix}
	e^{\hat{\boldsymbol{\phi}}_r}&\boldsymbol{p}_r\\\boldsymbol{0}_{1\times3}&1
\end{bmatrix}$$
where $\boldsymbol{\phi}_r=[1 \ -2 \ -3]^\top$ and $\boldsymbol{p}_r=[ 5 \ 10 \ 8]^\top$. 
The control inputs are set as the length of the 16 cables. 
Additionally, we include the term $\boldsymbol{q}^\top \boldsymbol{K}\boldsymbol{q}+ G$ in the objective function, which represents the elastic potential energy and gravitational potential energy of the soft manipulator. This term serves as a regularization term that improves the stability of the optimization and helps reduce ambiguity in the solution.
We set the prediction horizon as $\Delta t=0.06$s. Subsequently, this horizon is discretized into 6 steps, with each step having a duration of $0.01$s. 
\subsubsection{Control constraints}\label{B2}
Considering the practical physical limitations on tension and pulling speed of cable, the constraints for the NMPC controller are set as follows:
$$0\mathrm{N}\leq T_{1,2,\dots,16}\leq 20\mathrm{N}$$
$$-0.4\mathrm{m}/\mathrm{s}\leq \dot{l}_{1,2,\dots,16}\leq 0.4\mathrm{m}/\mathrm{s}$$
\begin{figure}[t]
	\centering
	\subfigure[The pose error of the end-effector ]{\includegraphics[width=0.48\textwidth]{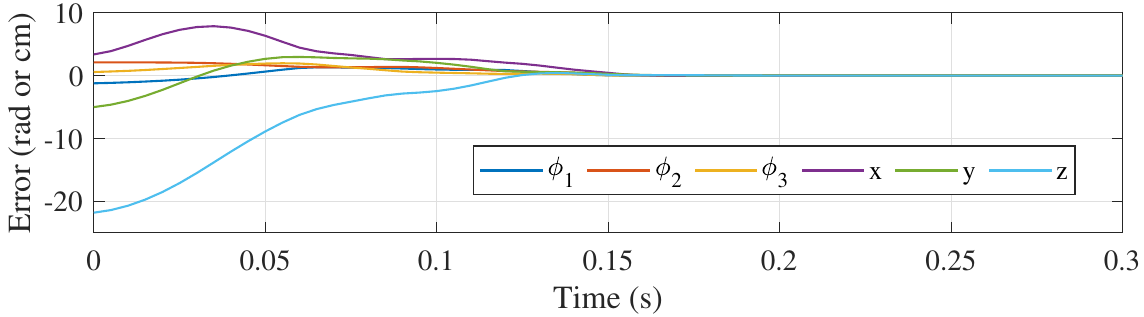}\label{fig:subfigAs}}
	\\
	\subfigure[The tension of each cable]{\includegraphics[width=0.48\textwidth]{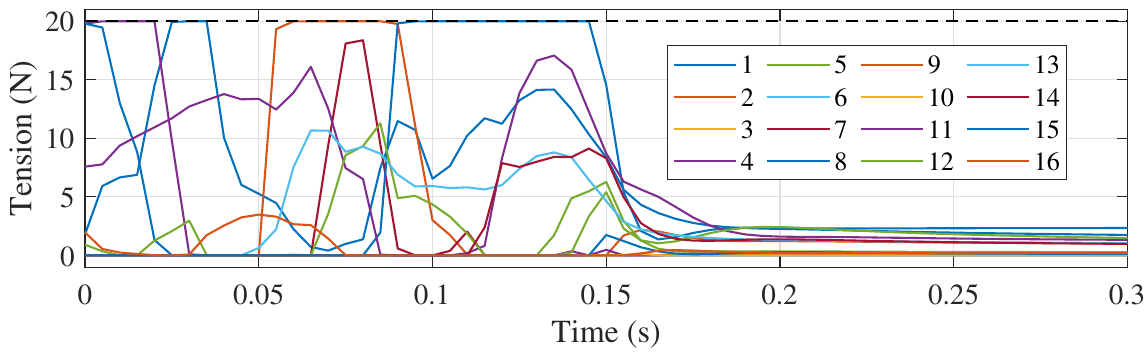}\label{fig:subfigBs}}
	\\
	\subfigure[The pulling speed of each cable]{\includegraphics[width=0.48\textwidth]{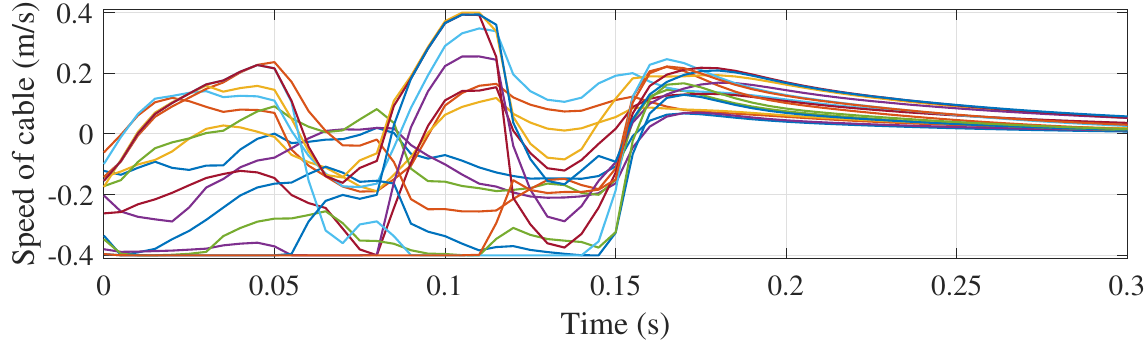}\label{fig:subfigCs}}
	\caption{The variations of control error, cable tension, and pulling speed over time during the control process.}
	\label{fig:mainfigs}
\end{figure}
\subsubsection{Analysis of simulation result}
Fig. \ref{constant} depicts the configurations of the soft manipulator at various time instances, along with its predicted configurations within the prediction time window. The variation of control error over time is shown in Fig. \ref{fig:subfigAs}. As can be seen from the figure, the system's response time for the given initial and target pose is $0.14$s and its final configuration stabilizes after $0.2$s. Fig. \ref{fig:subfigBs} shows the tension of the cable during the control process. It can be observed that the tension of all cables is constrained between $0$N and $20$N and eventually stabilizes. Similarly, Fig. \ref{fig:subfigCs} shows the pulling speed of the cable during the control process. It can be seen that all cable pulling speeds are constrained between $-0.4$m/s and $0.4$m/s, in accordance with the controller constraints we set.
\subsection{Trajectory tracking control}
In this subsection, we will introduce the second control test, which focuses on controlling the end-effector to accurately track a time-varying trajectory, encompassing both orientation and position.
\subsubsection{Control objective}
As illustrated in Fig. \ref{fig:mainfig}, at the initial moment, the soft manipulator is solely influenced by gravity. The predefined trajectory of target pose is set as follows:
$$\boldsymbol{g}_r=\begin{bmatrix}
	e^{\tilde{\boldsymbol{\phi}}_r}&\boldsymbol{p}_r\\\boldsymbol{0}_{1\times3}&1
\end{bmatrix}$$
where
$\boldsymbol{\phi}_r=[0 \ 4\sin(2t) \ 2\cos(2t)]^\top$
and
$\boldsymbol{p}_r=[10+3\sin(2t) \ 5\sin(2t) \ -7\cos(4t)]^\top$.  To investigate the impact of different prediction horizon lengths on control performance, we selected two different prediction horizons: $\Delta t=0.01$s and $\Delta t=0.06$s. Both prediction horizons have a discrete time step of $0.01$s. 
\subsubsection{Control constraints}
Similarly, the control constraints set here are the same as those set in Section \ref{B2}.
\begin{figure}[t]
	\centering
	\subfigure[Predicted horizon $\Delta t = 0.01$s ]{\includegraphics[width=0.49\textwidth]{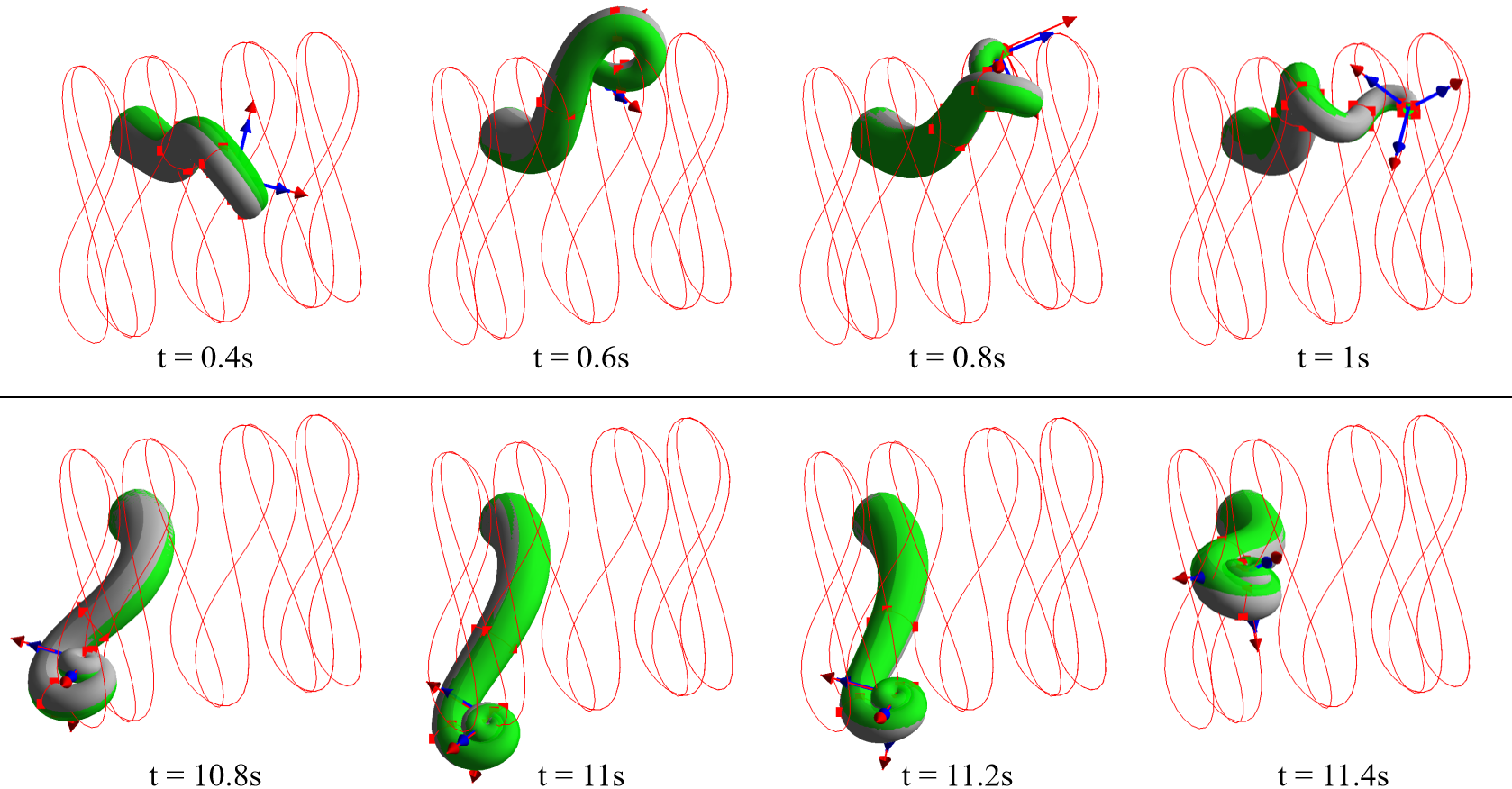}
		\label{fig:subfigA}}
	\\
	\subfigure[Predicted horizon $\Delta t = 0.06$s]{\includegraphics[width=0.49\textwidth]{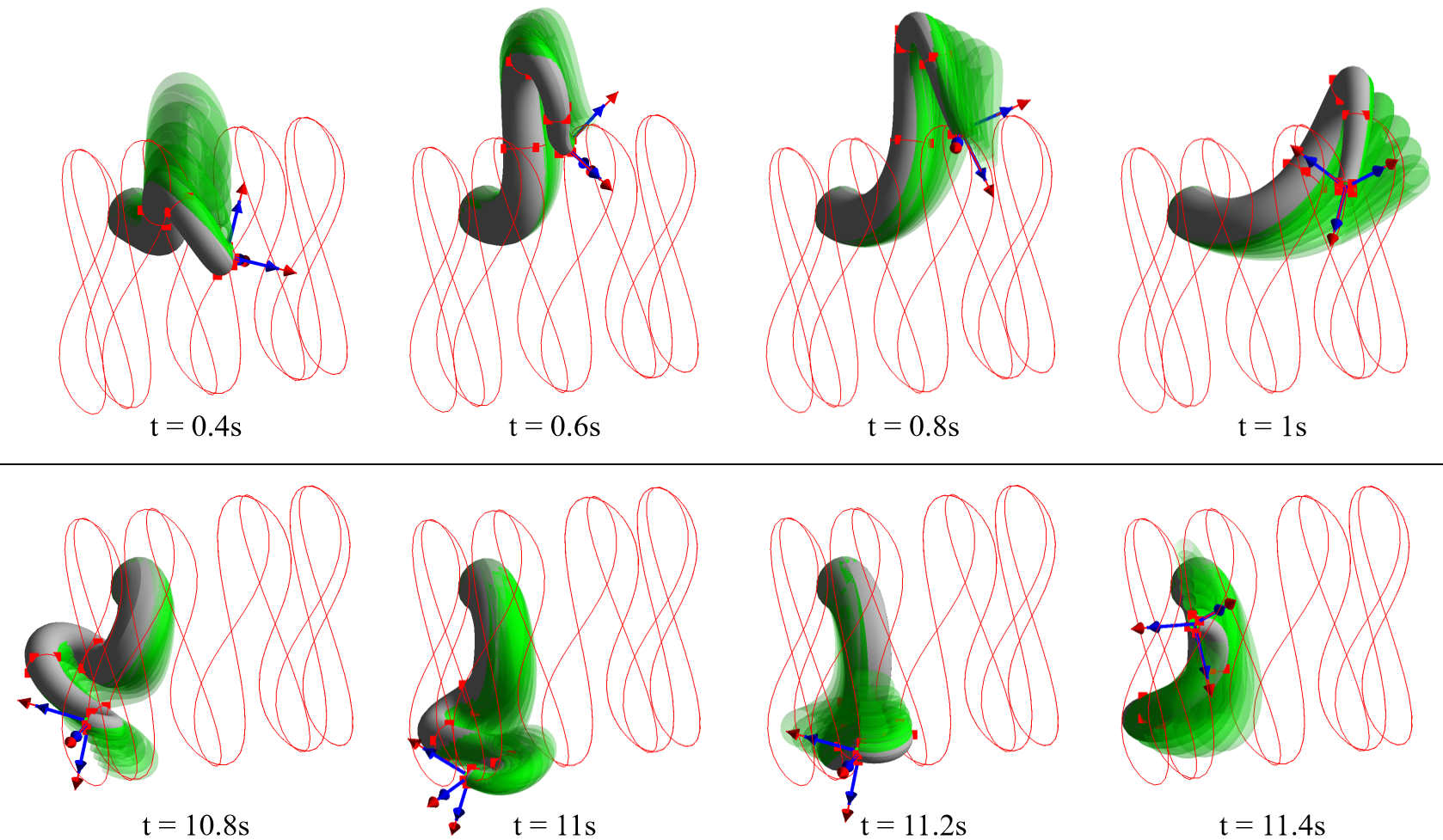}
		\label{fig:subfigB}}
	\caption{The configurations of the soft robot (gray rod) at various time instances during the control process, along with their corresponding predicted configurations (green rod) within the prediction horizon, are illustrated. The figure shows the predicted configurations at each discretized time point within the prediction horizon, with a time interval of $0.01$s  between each pair of adjacent predicted configurations. The simulation time step is set to $0.01$s.}
	\label{fig:mainfig}
\end{figure}
\begin{figure}[t]
	\centering
	\subfigure[Predicted horizon $\Delta t = 0.01$s]{\includegraphics[width=0.49\textwidth]{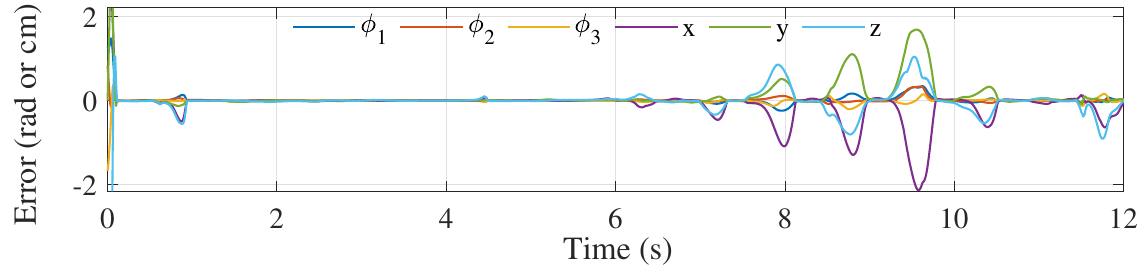}}
	\label{fig:subfigAq}
	\\
	\subfigure[Predicted horizon $\Delta t = 0.06$s]{\includegraphics[width=0.49\textwidth]{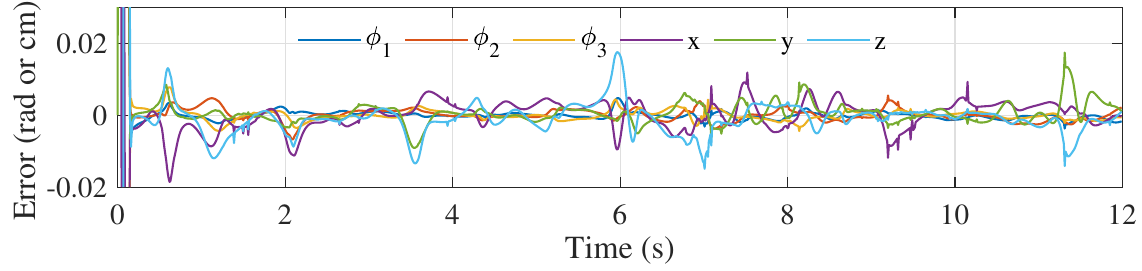}}
	\label{fig:subfigBq}
	\caption{The variations of control error over time during the control process.}
	\label{fig:mainfigq}
\end{figure}
\subsubsection{Influence of the duration of prediction horizon}
Fig. \ref{fig:mainfig} shows the real configuration and the predicted configurations of the soft manipulator during the tracking control. 
Fig. \ref{fig:mainfigq} illustrates the trajectory tracking error. When the prediction horizon is 0.01 seconds, the simulation results show that the soft manipulator deviates from the target trajectory at certain points. The short prediction horizon prevents the soft manipulator from anticipating future trajectory changes, making it unable to adjust its current pose to accommodate the future target trajectory. Consequently, the manipulator adopts an unreasonable pose in its attempt to follow the target position.

As shown in Fig. \ref{fig:subfigA}, as the end-effector moves along the target trajectory, the distal segment of the soft manipulator becomes increasingly curled. This happens because the optimization algorithm, constrained by the short prediction horizon, falls into a local optimum at the current moment, making the manipulator "short-sighted." In contrast, this issue is resolved in Fig. \ref{fig:subfigB}. With the prediction horizon extended to $0.06$s, the soft manipulator can plan its pose over the entire prediction horizon, allowing it to make present adjustments that better align with the future trajectory. Compared to the first test, the soft manipulator in this test exhibits a smoother pose adjustment in response to changes along the trajectory, and the curling phenomenon at the distal end no longer occurs. As a result, Fig. \ref{fig:mainfigq} shows that expanding the prediction horizon significantly reduces the tracking error for the test group with $t = 0.06$s compared to the test with $t = 0.01$s.
\subsection{Strain-pose coupled control}
\begin{figure}[t]
	\centering
	\includegraphics[width=0.45\textwidth]{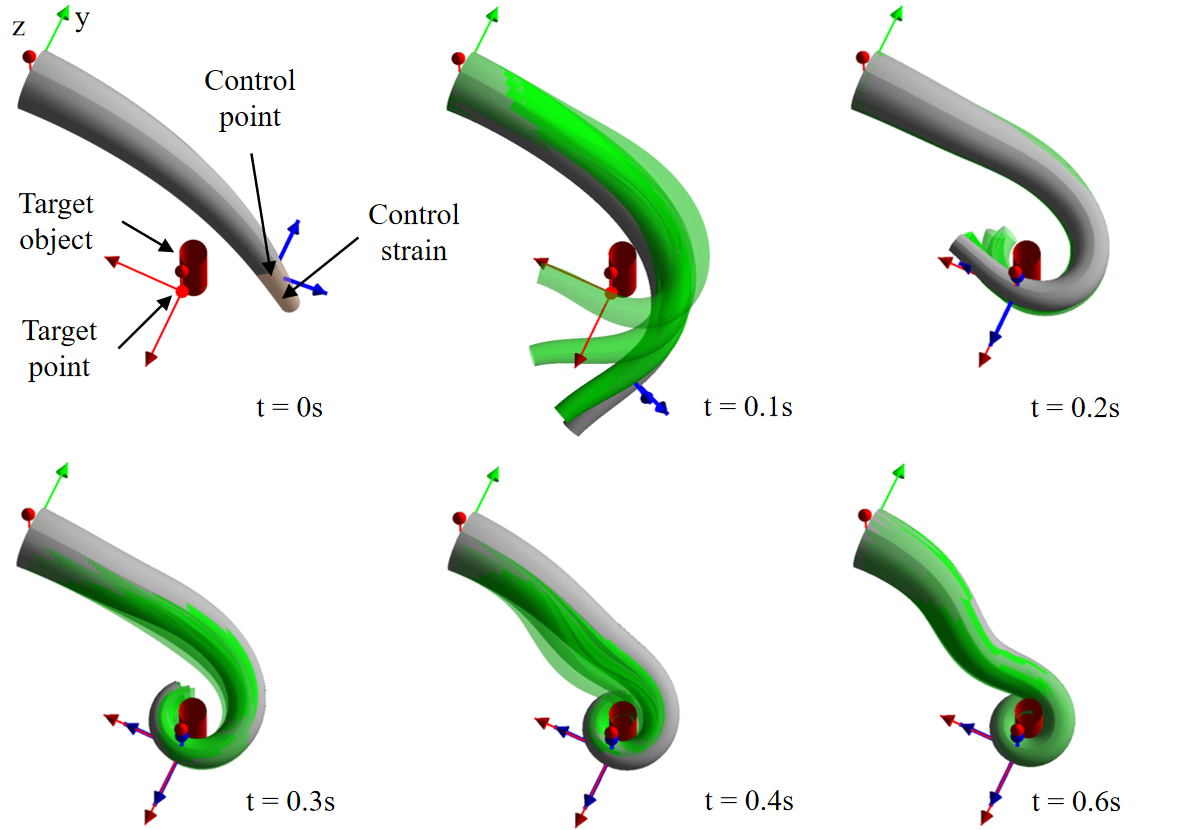}
	\caption{The configuration of the soft manipulator while grasping the target object at different time steps is illustrated. The gray color represents the current configuration of the soft manipulator, the green color indicates the predicted configuration, and the red color denotes the target object. The figure shows the predicted configurations at each discretized time point within the prediction horizon, with a time interval of $0.025$s  between each pair of adjacent predicted configurations. The simulation time step is set to $0.01$s.}
	\label{constant4}
\end{figure}
\begin{figure}[t]
	\centering
	\subfigure[Error of pose]{\includegraphics[width=0.49\textwidth]{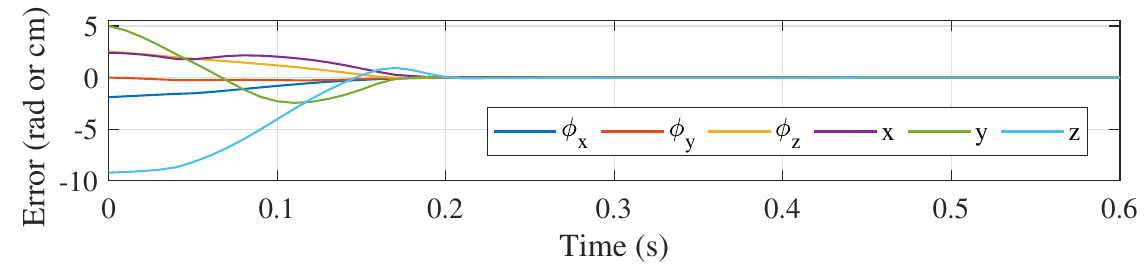}}
	\label{fig:subfigAqa}
	\\
	\subfigure[Error of strain]{\includegraphics[width=0.49\textwidth]{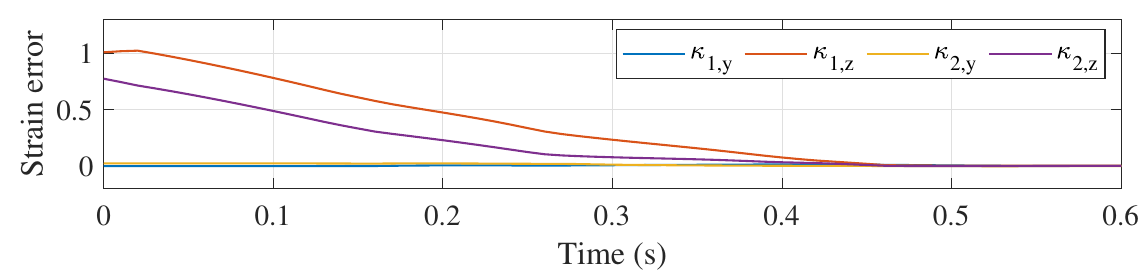}}
	\label{fig:subfigBqa}
	\caption{The variations of control error over time during the control process.}
	\label{fig:mainfigqa4}
\end{figure}
\subsubsection{Control objective}
In this set of tests, we simultaneously control the pose and strain of the soft manipulator with the goal of achieving a grasping action. The target object is a vertically placed cylinder with a diameter of 1.2 cm along the z-axis in global frame, with its axis's projection in the xy-plane at coordinates $(10,-3.75)$. When the soft manipulator attempts to grasp an object, we decompose the process into two simultaneous steps. The first step involves moving the distal segment of the soft manipulator to the vicinity of the target object. For this step, we select a control point $\boldsymbol{p}_c$ located 4 cm from the end-effector, whose configuration tensor is denoted as $\boldsymbol{g}_c$. 

At the same time, the second step involves inducing a bend in the distal end of the soft manipulator to wrap around the target object. In this step, the control focus is on the bending strain within 4 cm of the end-effector. Since we adopt PLS method for strain field interpolation, we target the bending strains at the last two interpolation nodes, specifically denoted as $\boldsymbol{\xi}_c=[\kappa_{1,y} \ \kappa_{1,z} \ \kappa_{2,y} \ \kappa_{2,z}]^\top$.  Following this, the optimization objective is set as:
$$\arg\min\limits_{\dot{l}_{1,2,\dots,16}} \int_{t}^{t+\Delta t}
\boldsymbol{W}
\begin{bmatrix}
    \mathcal{J}_{\boldsymbol{g}}(\boldsymbol{g})\\
	\mathcal{J}_{\boldsymbol{\xi}}(\boldsymbol{q})\\
	\boldsymbol{q}^\top \boldsymbol{K}\boldsymbol{q}+G
\end{bmatrix}\mathrm{d}t$$
where $\mathcal{J}_{\boldsymbol{g}}(\boldsymbol{g})$ is defined in \eqref{jg}. $\mathcal{J}_{\boldsymbol{\xi}}(\boldsymbol{q})$ are defined as $\Vert \boldsymbol{\xi}_c-\boldsymbol{\xi}_r\Vert_2$. $\boldsymbol{\phi}_r=[0 \ 0 \ -\pi]^\top$, $\boldsymbol{p}_r=[ 10 \ -5 \ -5]^\top$ and  $\boldsymbol{\xi}_r=[0 \ -0.8 \ 0 \ -0.7]^\top$. We set the weight matrix $\boldsymbol{W}$ as $[3 \ 1 \ 0.001 ]$.
\subsubsection{Control constraints}
The control constraints set here are the same as those set in Section \ref{B2}.
\subsubsection{Analysis of simulation result}
Fig. \ref{constant4} depicts the configurations of the soft manipulator at various time instances, along with its predicted configurations within the prediction horizon. The soft manipulator first moves to the side of the target object. Then, the distal end bends until it achieves the target bending strain, securely locking onto the object. The variation of control error over time are shown in Fig. \ref{fig:mainfigqa4}. Fig. \ref{fig:mainfigqa4} demonstrates that, under feedback control, the soft manipulator can ultimately converge to both the target pose and the target strain. As can be seen from the figure, the system's response time for the target pose and target strain is $0.17$s and $0.4$s respectively. The final configuration of soft manipulator stabilizes after $0.5$s. 
\section{Experiments}\label{sec:8}
This section presents the experimental validation of the proposed MHE-NMPC framework on a cable-driven soft manipulator prototype. Unlike the 16-cable manipulator considered in the numerical simulations, the physical prototype used here is actuated by four cables and is therefore employed to evaluate the real-time feasibility of the proposed framework for task-space position control. Specifically, the experiments focus on end-effector position tracking, which is one of the most common control objectives for soft manipulators in practical applications. The following subsections report the validation of the MHE-based state estimator, the end-effector tracking performance under different reference trajectories, and the real-time marker-tracking experiments.
\subsection{Experimental setup}
\subsubsection{Soft manipulator}
The main body of the soft manipulator consists of a silicone cone. As illustrated in Fig. \ref{prototype}, rigid sleeves of varying diameters are distributed at equal intervals along the axial direction of the silicone structure. Four cables are attached to the terminal sleeve at one end, pass through the intermediate sleeves, and are ultimately connected to an external actuation device. These rigid sleeves are designed to bear the contact forces exerted by the cables and transfer these forces to the silicone body of the soft manipulator. The identified material parameters of the silicone are as follows: Young's modulus \( E = 2.563 \times 10^5 \) Pa, shear modulus \( G = 8.543 \times 10^4 \) Pa, and material density \( \rho = 1.41 \times 10^3 \) kg/m\(^3\).
\subsubsection{Control platform}
The control platform comprises three main components: 4 stepper motors for driving the soft manipulator, magnetic sensor (LIBERTY 240/16 base system) for detecting the end-effector position, and a computer. As illustrated in  Fig \ref{prototype2}, four stepper motors are mounted on the experimental base, each controlling one of the four cables in the soft manipulator. The magnetic sensor is positioned on the manipulator's end-effector, with its location detected by a magnetic receiver situated next to the manipulator. The receiver processes the position data and sends it to the computer. The computer then performs the algorithm for the observer and controller, and transmits the drive signals to the stepper motors.
\subsection{Experimental tests}
In this subsection, we concentrate on the control experiments conducted with the soft manipulator. First, we assess the reliability of the receding horizon observer. Following this, we design the closed-loop experiments that integrates the receding horizon observer with model predictive control. Finally, we evaluate the tracking performance of the soft manipulator's end-effector in two distinct scenarios: one where it follows a predetermined target trajectory, and another where it tracks a randomly varying target trajectory.
\begin{figure}[t]
	\centering
	\includegraphics[width=0.48\textwidth]{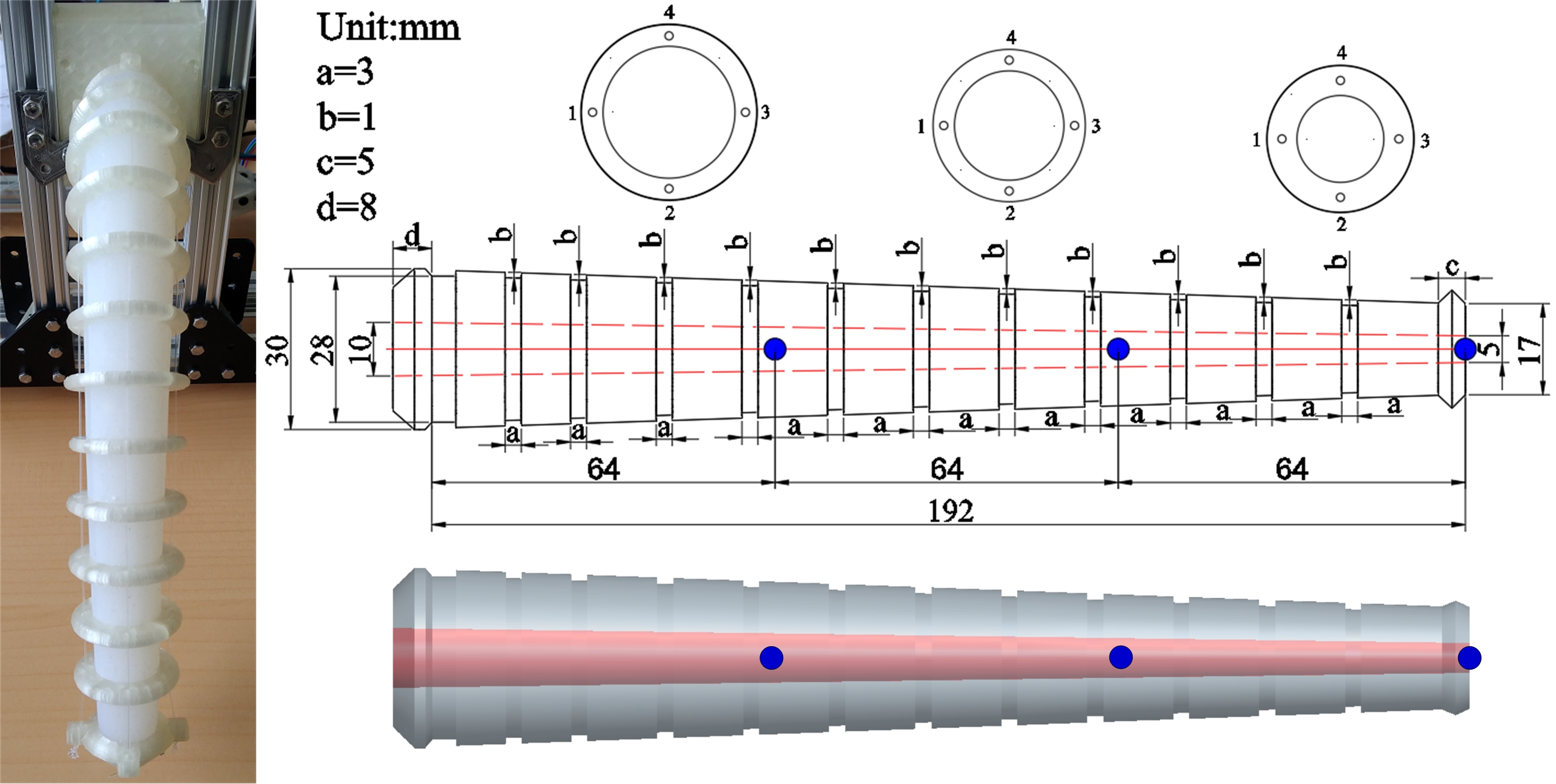}
	\caption{Prototype of soft manipulator and its dimensional diagram.}
	\label{prototype}
\end{figure}
\begin{figure}[t]
	\centering
	\includegraphics[width=0.48\textwidth]{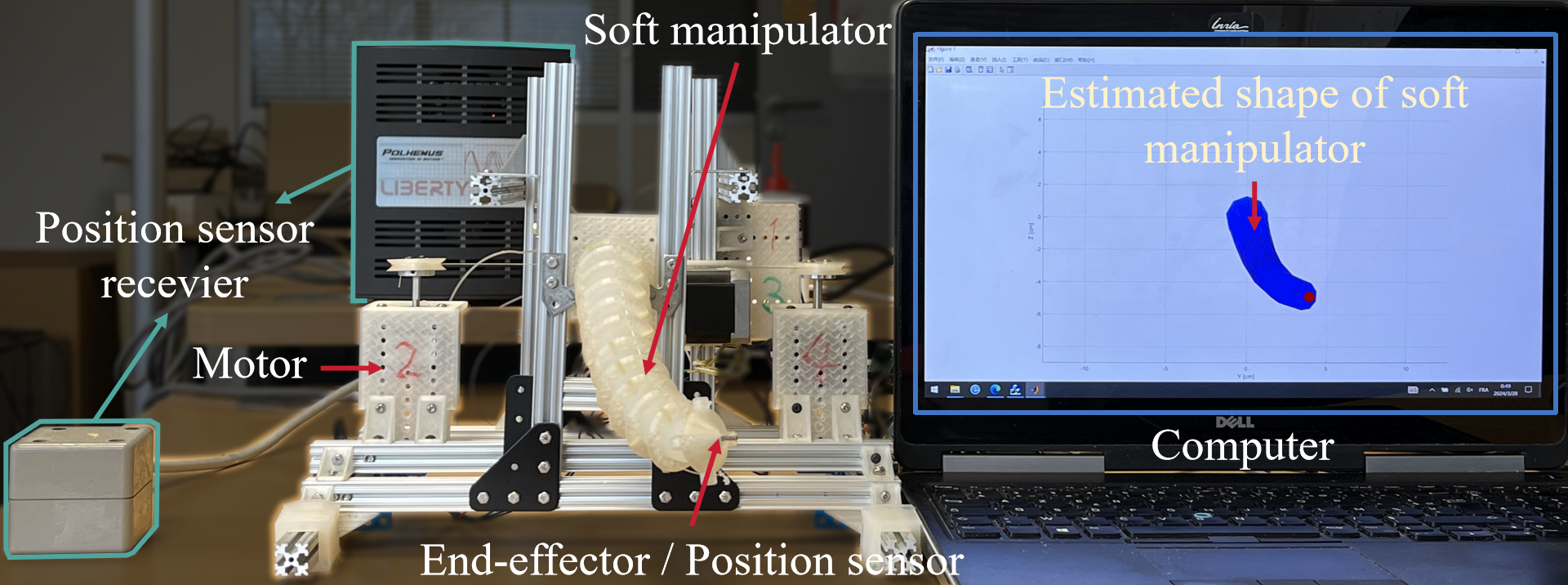}
	\caption{Experimental platform for implementing closed-loop control based on MHE-NMPC.}
	\label{prototype2}
\end{figure}
\subsubsection{MHE based observer}
In this experiment, the operator randomly pulls four cables in the experimental setup to generate the movement of the soft manipulator. The moving horizon estimator estimates the shape of the manipulator based on the measured trajectory of the end-effector. We set the time horizon of MHE as $0.3$s. Subsequently, this horizon is discretized into 3 steps, with each step having a duration of $0.1$s. 
A comparison has been made between the real configuration of the soft manipulator and those derived from the state estimation. As shown in \cref{obs} and \cref{observer}, the proposed MHE-based estimator can track the end-effector position of the soft manipulator within a finite time. The estimation error in \cref{observer} converges rapidly and remains bounded during the motion. Using the estimated reduced state, the manipulator configuration can be reconstructed, and the visual comparison in \cref{obs} qualitatively confirms the consistency between the estimated shape and the observed deformation of the physical prototype.
\begin{figure}[t]
	\centering
	\includegraphics[width=0.49\textwidth]{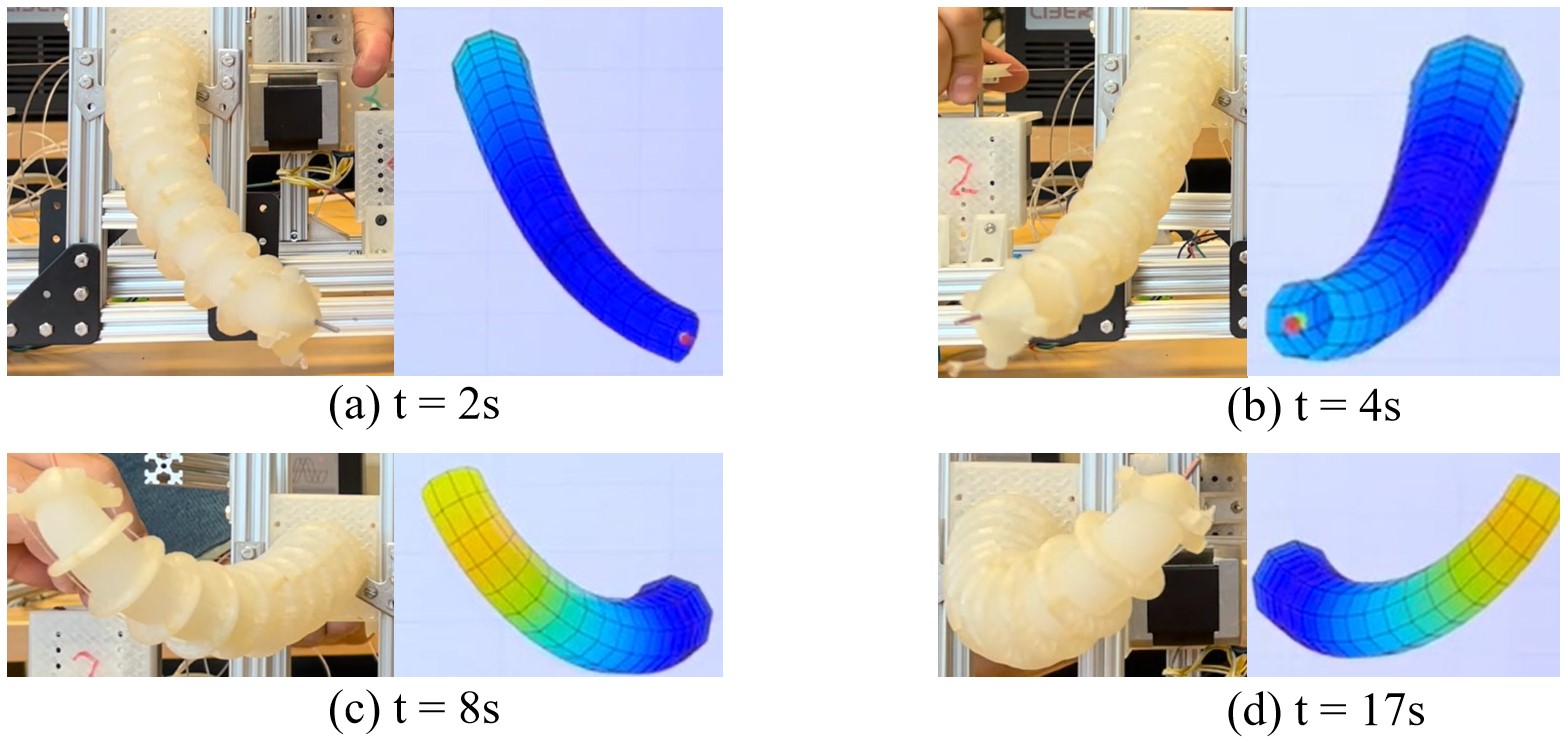}
	\caption{Comparison between the real shape and the estimated shape of soft manipulator.}
	\label{obs}
\end{figure}
\begin{figure}[t]
	\centering
	\includegraphics[width=0.49\textwidth]{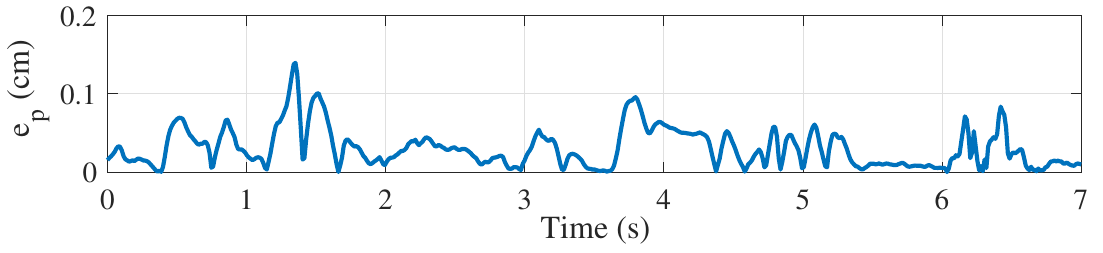}
	\caption{The variations of estimated error (defined as the distance between the estimated and actual positions of the end-effector) over time.}
	\label{observer}
\end{figure}
\subsubsection{Trajectory tracking control}
In this experiment, the control objective for the soft manipulator is to enable the end-effector to accurately follow a predefined target trajectory. Given that the soft manipulator is actuated by only four cables, the workspace of the end-effector constitutes a surface with minimal thickness. This limited thickness is due to the negligible axial compression strain experienced by the soft manipulator when tensioned by the cables, as noted in \cite{amehri2021workspace}. To prevent the target trajectory from extending beyond the available workspace, we manage only two of the three spatial coordinates (X, Y, Z) of the end-effector's position. For the purpose of assessing the effectiveness of the control algorithm, two target trajectories are established—one in the Y-Z plane and another in the X-Z plane.

In the first experiment, the target trajectory is set in the Y-Z plane, and its mathematical expression is given by: $$Y=6.75\sin(0.9t) \ , \ Z=-6.75\cos(3.6t)-1.35$$

In the second experiment, the target trajectory is set in the x-z plane, and its mathematical expression is given by: $$X=12.5+3.5\sin(3t) \ , \ Z=-4\sin(6t)-1.5$$

To evaluate the tracking performance, various controllers were tested on the soft manipulator using the same reference trajectories. Among these, the proposed estimation-based controller was compared with the static model-based controller (SMBC) \cite{li2023discrete}. We set the prediction horizon of NMPC as $\Delta t=0.3$s. Subsequently, this horizon is discretized into 3 steps, with each step having a duration of $0.1$s. 

\cref{fig:track} depicts the tracking performance of the NMPC and SMBC controllers across rapidly changing trajectories. A detailed analysis was conducted to assess the tracking accuracy, utilizing metrics such as the absolute maximum error (AVME) and the root mean square error (RMSE) of the end-effector's position. The results highlighting the performance differences between the controllers are documented in \cref{tab1}.

\cref{fig:track} reveals that under the SMBC controller, the soft manipulator only approximately follows the designated trajectories, particularly struggling with rapid circular movements due to a failure to consider the model's velocity and acceleration. In contrast, the proposed NMPC controller achieves improved trajectory-tracking accuracy, as reflected by the lower AVME and RMSE values reported in \cref{tab1}. Compared with the SMBC controller, the NMPC controller reduces the tracking errors along the corresponding axes for circular trajectories in both the (Y)-(Z) and (X)-(Z) planes. These results indicate that the proposed controller provides more accurate tracking performance and is feasible for controlling the soft manipulator during rapid trajectory-following tasks.
\begin{figure}[t]
	\centering
	\subfigure[Y-Z plane]{\includegraphics[width=0.49\textwidth]{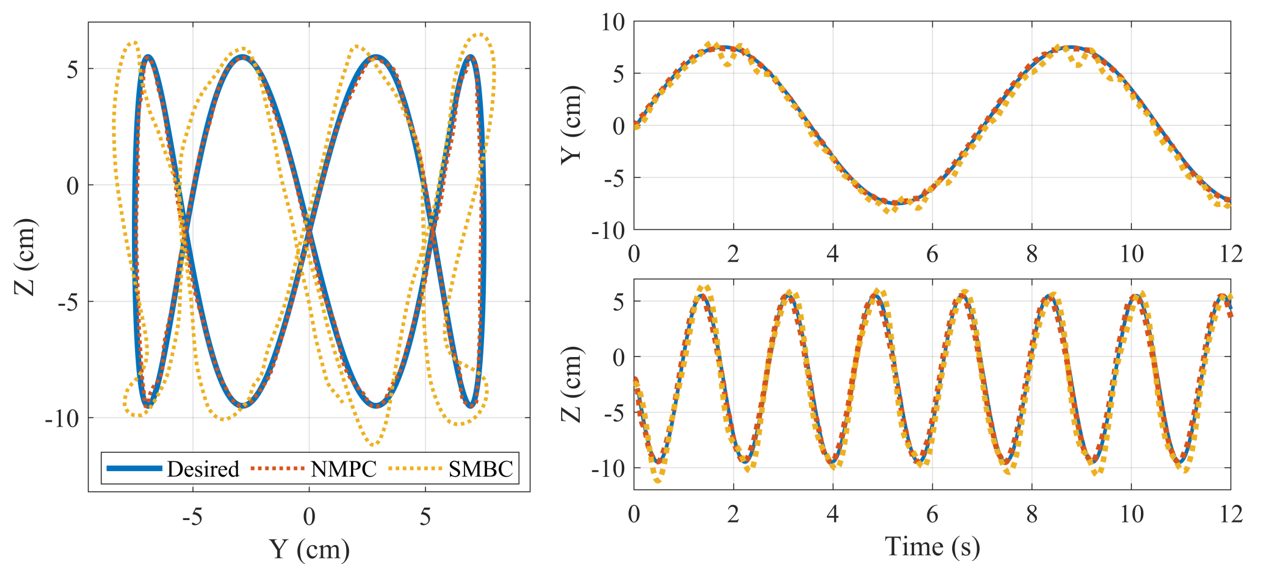}}
	\label{trackyz}
	\\
	\subfigure[X-Z plane]{\includegraphics[width=0.49\textwidth]{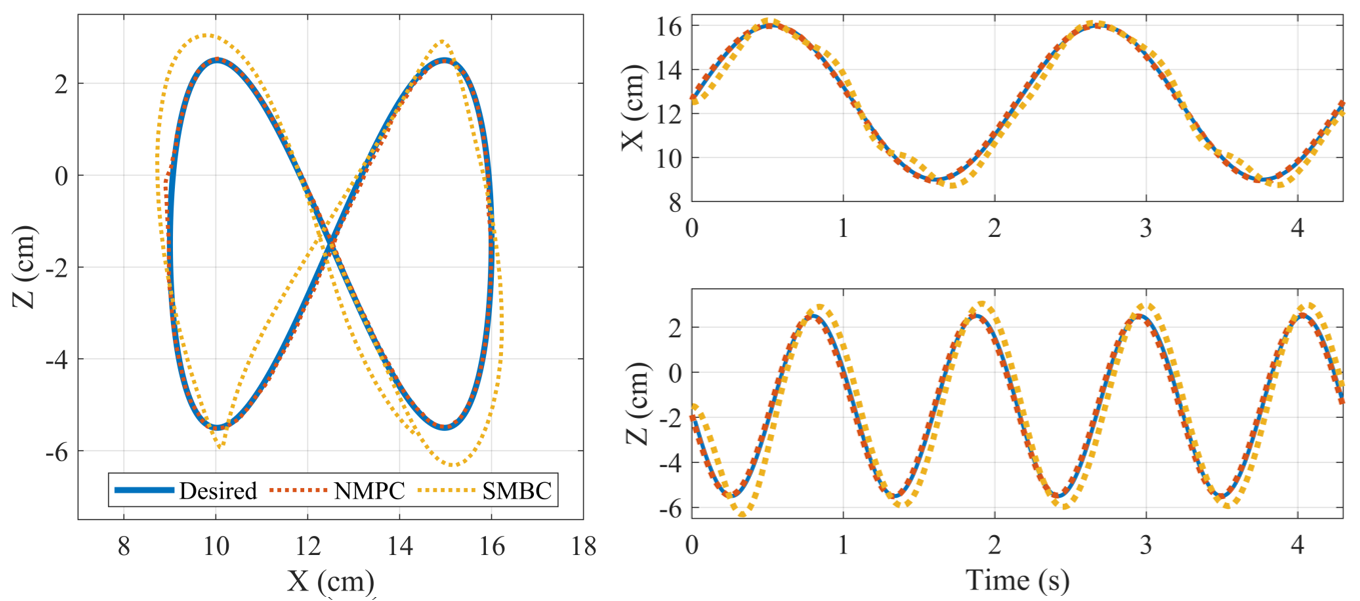}}
	\label{trackxz}
	\caption{The variations of control error over time during the control process.}
	\label{fig:track}
\end{figure}
\begin{table}[t]
		\sisetup{
				table-number-alignment = center,
				table-figures-integer = 1,
				table-figures-decimal = 4
			}
	\caption{Quantitative  Analysis of End-Effector Position Tracking Control in X, Y, and Z Directions for Various Controllers.}
	\label{tab1}
	\centering
	\setlength{\tabcolsep}{1.2pt}
	\begin{tabular}{*{8}{S}}
		\toprule  
		{\multirow{2.5}{*}{\stackanchor{Trajectory}{plan}}} & {\multirow{2.5}{*}{\stackanchor{Control}{strategy}}} & \multicolumn{2}{c}{$e_x$(cm)} & \multicolumn{2}{c}{$e_y$(cm)} & \multicolumn{2}{c}{ $e_z$(cm)}\\
		\cmidrule(r){3-8}
		&&{AVME}&{RMSE}&{AVME}&{RMSE}&{AVME}&{RMSE}
		\\\midrule 
		{\multirow{2}{*}{Y-Z plane}} & 
		  {SMBC} & {-} & {-} & 2.274 & 1.044 & 2.352 & 1.216 \\
		& {NMPC} & {-} & {-} & 0.108 & 0.064 & 0.115 & 0.083\\
		\midrule
		{\multirow{2}{*}{X-Z plane}} & 
		  {SMBC} & 0.771 & 0.310 & {-} & {-} & 1.419 & 0.596 \\
		& {NMPC} & 0.095 & 0.062 & {-} & {-} & 0.086 & 0.047\\
		\bottomrule
	\end{tabular}
\end{table}

\begin{table}[t]
\caption{Computation Time of the Proposed MHE--NMPC Framework in the Experiments.}
\label{tab:computation_time}
\centering
\renewcommand{\arraystretch}{1.18}
\setlength{\tabcolsep}{4.0pt}
\begin{tabular}{llccc}
\toprule
Experiment & Module & Horizon & Mean SQP iter. & Mean time \\
\midrule
\multirow{2}{*}{\shortstack{Trajectory\\tracking}}
& MHE  & 0.3 s & 2.2 & 18.6 ms \\
& NMPC & 0.3 s & 3.5 & 32.8 ms \\
\midrule
\multirow{2}{*}{\shortstack{Marker\\tracking}}
& MHE  & 0.3 s & 2.3 & 19.2 ms \\
& NMPC & 0.4 s & 3.7 & 41.2 ms \\
\bottomrule
\end{tabular}
\end{table}


\subsubsection{Real time marker tracking}
In some scenarios, the target trajectory is not predetermined but changes in real-time, such as when a soft manipulator interacts with its environment. In this experiment, we control the soft manipulator to follow a randomly varying trajectory in real-time. As illustrated in Fig. \ref{markertrack}, the experiment involves an operator holding a wand with position markers. As the operator waves the wand, the soft manipulator must track the wand's movement in real-time. Since the marker on the wand is located outside the workspace of the soft manipulator, we project both the coordinates of the wand's marker and the end-effector of the soft manipulator onto an imaginary work surface, and then control the projected position to track the projected marker position. 

Due to the random and rapid changes in the position of the target point, the control input may experience significant variations in a short time. This phenomenon can reduce the smoothness of control and cause the numerical optimization solution to fall into undesirable local extrema. Additionally, the system hardware has performance limits, such as the rotational speed and output torque of the stepper motors. Moreover, excessive cable contraction speed can generate substantial friction, increasing system error and reducing robustness. Therefore, to enhance control robustness and meet hardware requirements, we introduce constraints on cable speed in the control optimization problem. The constrained control objective is defined as follows:
$$\begin{aligned}
	&\min && \int_{t}^{t+\Delta t} \Vert\boldsymbol{p}_{proj}-\boldsymbol{p}_{r,proj}\Vert \mathrm{d}t\\
    &\mathrm{subject \ to}
    &&-2\pi r_{d}{w}_{max}\leq\dot{l}_{1,2,3,4}\leq 2\pi r_{d}{w}_{max}
\end{aligned}$$
where the $r_d = 1 \mathrm{cm}$ is the radius of the rope pulley, ${w}_{max} = 5\mathrm{rps}$ is the maximum allowable rotational speed of the motor. We set the prediction horizon of NMPC as $\Delta t=0.4$s and the horizon is discretized into 4 steps, with each step having a duration of $0.1$s. 
\begin{figure}[t] 
	\centering
	\includegraphics[width=0.49\textwidth]{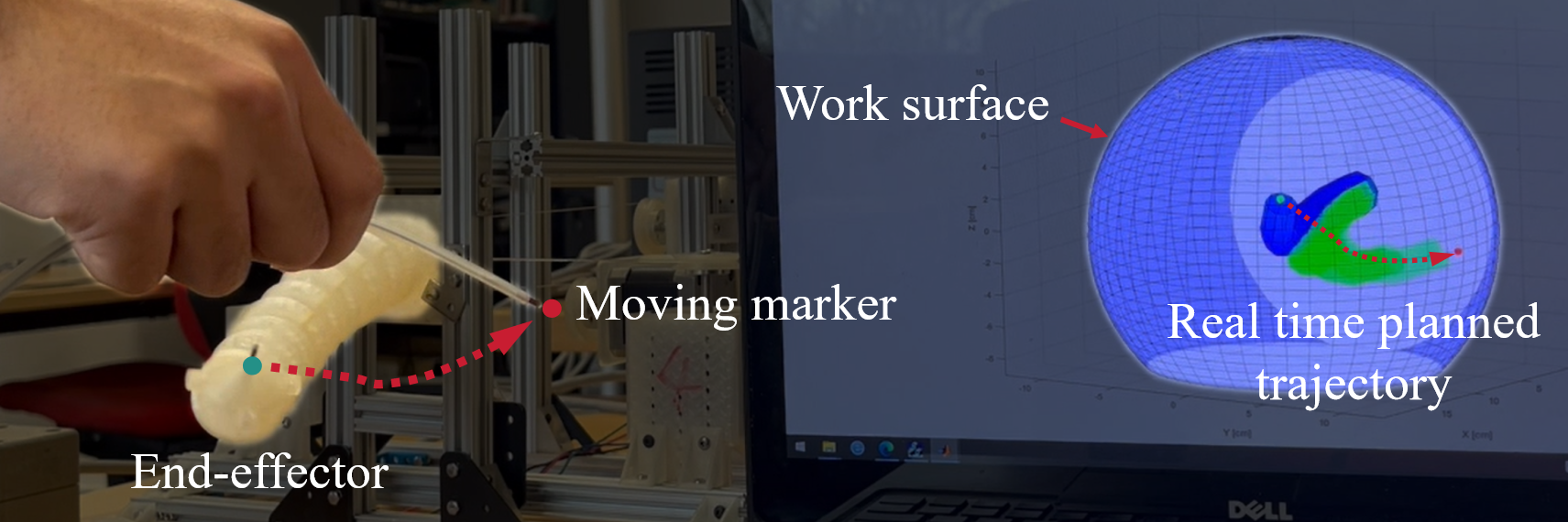}
	\caption{Experimental setup for real-time marker tracking.}
	\label{markertrack}
\end{figure}
\begin{figure}[t]
	\centering
	\includegraphics[width=0.49\textwidth]{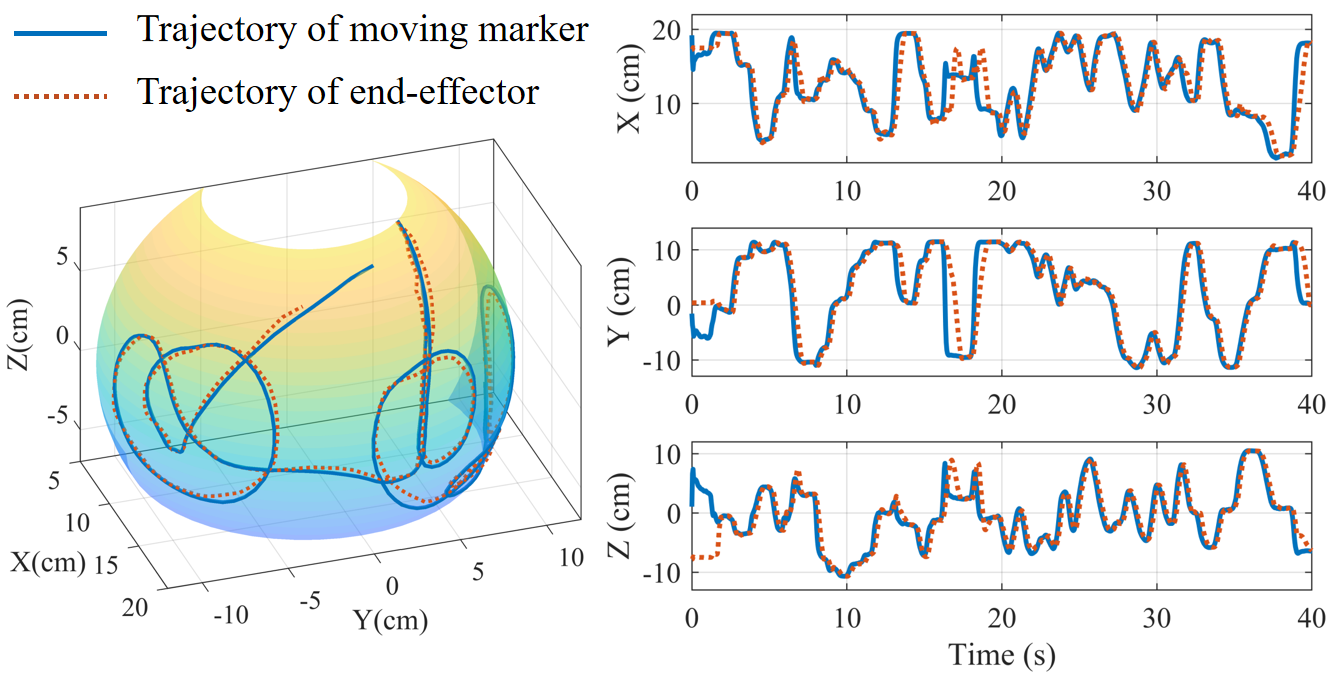}
	\caption{Real-time marker tracking results: projected target trajectory and end-effector trajectory.}
	\label{markertrack3}
\end{figure}

\cref{markertrack3} illustrates the end-effector tracking results during real-time marker tracking. The proposed controller continuously updates the planned motion according to the dynamically changing target position, enabling the end-effector to follow the target within the reachable workspace. Quantitatively, the tracking error has an RMSE of $0.42\, $cm over the entire experiment. These results demonstrate the feasibility of the proposed MHE-NMPC framework for real-time task-space tracking of cable-driven soft manipulators.

\subsubsection{Computational Performance}

Finally, we evaluated the computational performance of the proposed MHE--NMPC framework during the experiments. The sampling period was set to 0.1 s, and the computation time of the MHE and NMPC modules was recorded over all control iterations.
As summarized in Table~\ref{tab:computation_time}, the average total computation time is approximately 51.4 ms for trajectory tracking and 60.4 ms for marker tracking. Both are below the sampling period, indicating that the proposed framework can be executed online for real-time end-effector tracking on the physical prototype.

\section{Conclusions}\label{sec:9}
In this work, we presented a unified modeling, estimation, and control framework for cable-driven soft manipulators. To account for the unilateral tension--slackness behavior of cables, the cable actuation mechanism was first formulated using complementarity conditions. These conditions were then approximated by smooth equality constraints through the introduction of slack variables, allowing cable-length inputs to be incorporated into a differentiable optimization framework. Based on this formulation, analytical Jacobians of the manipulator dynamics and cable-actuation constraints were derived to support efficient simulation, estimation, and optimization.

Building on the reduced Cosserat-rod dynamics, we developed an MHE-NMPC framework for task-space control and state estimation of cable-driven soft manipulators. The moving horizon estimator reconstructs the reduced state and the manipulator configuration from end-effector pose measurements and cable-length information, while the NMPC controller computes cable-length commands under cable-length and cable-rate constraints. Numerical simulations demonstrated the capability of the proposed framework for pose regulation, trajectory tracking, and strain-related control tasks on a multi-cable soft manipulator. Experimental results on a four-cable prototype further validated the real-time implementation of the framework and demonstrated accurate end-effector position tracking through cable-length control.

Overall, the proposed approach provides a practical and extensible foundation for model-based estimation and control of cable-driven soft manipulators. By combining reduced Cosserat dynamics, smooth cable-length modeling, and constrained optimization-based estimation and control, this work contributes toward real-time, task-space-oriented control of soft robotic systems. Future work will focus on improving quantitative shape-estimation validation, extending the experimental platform toward higher-dimensional pose regulation, and enhancing the computational efficiency and robustness of the proposed MHE-NMPC framework.



\appendix[]
	\subsection{Adjoint representation of the Lie algebra}\label{notations}
	The adjoint representation $\mathrm{ad}_{(\cdot)}$ of the Lie algebra is given by	
	$$
	{\rm{ad}}_{\boldsymbol{\xi}}= \left(\begin{matrix}
		\tilde{\boldsymbol{\kappa}}&\boldsymbol{0}_{3\times3}\\\tilde{\boldsymbol{\epsilon}}&\tilde{\boldsymbol{\kappa}}
	\end{matrix}\right)\in \mathbb{R}^{6\times6}
		$$
		The adjoint representation $\mathrm{ad}_{(\cdot)}^\star$ of the Lie algebra is given by
		$${\rm{ad}}_{\boldsymbol{\xi}}^\star= \left(\begin{matrix}
			\tilde{\boldsymbol{\kappa}}&\tilde{\boldsymbol{\epsilon}}\\\tilde{\boldsymbol{\epsilon}}&\boldsymbol{0}_{3\times3}
		\end{matrix}\right)\in \mathbb{R}^{6\times6}$$
		\subsection{Transformation matrix}\label{notations2}
		The matrix transforming the velocity or acceleration twist from body frame to inertial frame is given by
		$${\rm{Ad}}_{\boldsymbol{g}}= \left(\begin{matrix}
			\boldsymbol{R}&\boldsymbol{0}_{3\times3}\\\tilde{\boldsymbol{p}}\boldsymbol{R}&\boldsymbol{R}
		\end{matrix}\right)\in \mathbb{R}^{6\times6}.$$

\bibliographystyle{IEEEtran} 
\bibliography{references}

@article{whitesides2018soft,
  title={Soft robotics},
  author={Whitesides, George M},
  journal={Angewandte Chemie International Edition},
  volume={57},
  number={16},
  pages={4258--4273},
  year={2018},
  publisher={Wiley Online Library}
}

@article{hawkes2021hard,
  title={Hard questions for soft robotics},
  author={Hawkes, Elliot W and Majidi, Carmel and Tolley, Michael T},
  journal={Science robotics},
  volume={6},
  number={53},
  pages={eabg6049},
  year={2021},
  publisher={American Association for the Advancement of Science}
}

@article{chen2020design,
  title={Design optimization of soft robots: A review of the state of the art},
  author={Chen, Feifei and Wang, Michael Yu},
  journal={IEEE Robotics \& Automation Magazine},
  year={2020},
  publisher={IEEE}
}

@article{li2023piecewise,
  title={Piecewise Linear Strain Cosserat Model for Soft Slender Manipulator},
  author={Li, Haihong and Xun, Lingxiao and Zheng, Gang},
  journal={IEEE Transactions on Robotics},
  year={2023},
  publisher={IEEE}
}

@article{diehl2006fast,
  title={Fast direct multiple shooting algorithms for optimal robot control},
  author={Diehl, Moritz and Bock, Hans Georg and Diedam, Holger and Wieber, P-B},
  journal={Fast motions in biomechanics and robotics: optimization and feedback control},
  pages={65--93},
  year={2006},
  publisher={Springer}
}

@article{michalska1995moving,
  title={Moving horizon observers and observer-based control},
  author={Michalska, Hannah and Mayne, David Q},
  journal={IEEE Transactions on Automatic Control},
  volume={40},
  number={6},
  pages={995--1006},
  year={1995},
  publisher={IEEE}
}

@article{renda2014dynamic,
  title={Dynamic model of a multibending soft robot arm driven by cables},
  author={Renda, Federico and Giorelli, Michele and Calisti, Marcello and Cianchetti, Matteo and Laschi, Cecilia},
  journal={IEEE Transactions on Robotics},
  volume={30},
  number={5},
  pages={1109--1122},
  year={2014},
  publisher={IEEE}
}

@article{li2023discrete,
  title={Discrete Cosserat Static Model-Based Control of Soft Manipulator},
  author={Li, Haihong and Xun, Lingxiao and Zheng, Gang and Renda, Federico},
  journal={IEEE Robotics and Automation Letters},
  year={2023},
  publisher={IEEE}
}

@article{wachter2005line,
  title={Line search filter methods for nonlinear programming: Motivation and global convergence},
  author={W{\"a}chter, Andreas and Biegler, Lorenz T},
  journal={SIAM Journal on Optimization},
  volume={16},
  number={1},
  pages={1--31},
  year={2005},
  publisher={SIAM}
}

@article{boyer2020dynamics,
  title={Dynamics of continuum and soft robots: A strain parameterization based approach},
  author={Boyer, Frederic and Lebastard, Vincent and Candelier, Fabien and Renda, Federico},
  journal={IEEE Transactions on Robotics},
  volume={37},
  number={3},
  pages={847--863},
  year={2020},
  publisher={IEEE}
}

@article{allgower2004nonlinear,
  title={Nonlinear model predictive control: From theory to application},
  author={Allgower, Frank and Findeisen, Rolf and Nagy, Zoltan K and others},
  journal={Journal-Chinese Institute Of Chemical Engineers},
  volume={35},
  number={3},
  pages={299--316},
  year={2004},
  publisher={CHINESE INST CHEM ENGINEERS}
}

@article{dirkse1995path,
  title={The path solver: a nommonotone stabilization scheme for mixed complementarity problems},
  author={Dirkse, Steven P and Ferris, Michael C},
  journal={Optimization methods and software},
  volume={5},
  number={2},
  pages={123--156},
  year={1995},
  publisher={Taylor \& Francis}
}

@article{willems1971least,
  title={Least squares stationary optimal control and the algebraic Riccati equation},
  author={Willems, Jan},
  journal={IEEE Transactions on automatic control},
  volume={16},
  number={6},
  pages={621--634},
  year={1971},
  publisher={IEEE}
}

@book{nocedal1999numerical,
  title={Numerical optimization},
  author={Nocedal, Jorge and Wright, Stephen J},
  year={1999},
  publisher={Springer}
}

@article{marchese2016design,
  title={Design, kinematics, and control of a soft spatial fluidic elastomer manipulator},
  author={Marchese, Andrew D and Rus, Daniela},
  journal={The International Journal of Robotics Research},
  volume={35},
  number={7},
  pages={840--869},
  year={2016},
  publisher={SAGE Publications Sage UK: London, England}
}

@article{amehri2021workspace,
  title={Workspace boundary estimation for soft manipulators using a continuation approach},
  author={Amehri, Walid and Zheng, Gang and Kruszewski, Alexandre},
  journal={IEEE Robotics and Automation Letters},
  volume={6},
  number={4},
  pages={7169--7176},
  year={2021},
  publisher={IEEE}
}

@article{thuruthel2017learning,
  title={Learning dynamic models for open loop predictive control of soft robotic manipulators},
  author={Thuruthel, Thomas George and Falotico, Egidio and Renda, Federico and Laschi, Cecilia},
  journal={Bioinspiration \& biomimetics},
  volume={12},
  number={6},
  pages={066003},
  year={2017},
  publisher={IOP Publishing}
}

@inproceedings{gillespie2018learning,
  title={Learning nonlinear dynamic models of soft robots for model predictive control with neural networks},
  author={Gillespie, Morgan T and Best, Charles M and Townsend, Eric C and Wingate, David and Killpack, Marc D},
  booktitle={2018 IEEE International Conference on Soft Robotics (RoboSoft)},
  pages={39--45},
  year={2018},
  organization={IEEE}
}

@article{zheng2020robust,
  title={Robust control of a silicone soft robot using neural networks},
  author={Zheng, Gang and Zhou, Yuan and Ju, Mingda},
  journal={ISA transactions},
  volume={100},
  pages={38--45},
  year={2020},
  publisher={Elsevier}
}

@article{tariverdi2021recurrent,
  title={A recurrent neural-network-based real-time dynamic model for soft continuum manipulators},
  author={Tariverdi, Abbas and Venkiteswaran, Venkatasubramanian Kalpathy and Richter, Michiel and Elle, Ole J and T{\o}rresen, Jim and Mathiassen, Kim and Misra, Sarthak and Martinsen, {\O}rjan G},
  journal={Frontiers in Robotics and AI},
  volume={8},
  pages={631303},
  year={2021},
  publisher={Frontiers Media SA}
}

@article{hannan2003kinematics,
  title={Kinematics and the implementation of an elephant's trunk manipulator and other continuum style robots},
  author={Hannan, Michael W and Walker, Ian D},
  journal={Journal of robotic systems},
  volume={20},
  number={2},
  pages={45--63},
  year={2003},
  publisher={Wiley Online Library}
}

@article{della2020model,
  title={Model-based dynamic feedback control of a planar soft robot: trajectory tracking and interaction with the environment},
  author={Della Santina, Cosimo and Katzschmann, Robert K and Bicchi, Antonio and Rus, Daniela},
  journal={The International Journal of Robotics Research},
  volume={39},
  number={4},
  pages={490--513},
  year={2020},
  publisher={SAGE Publications Sage UK: London, England}
}

@article{della2020improved,
  title={On an improved state parametrization for soft robots with piecewise constant curvature and its use in model based control},
  author={Della Santina, Cosimo and Bicchi, Antonio and Rus, Daniela},
  journal={IEEE Robotics and Automation Letters},
  volume={5},
  number={2},
  pages={1001--1008},
  year={2020},
  publisher={IEEE}
}

@article{mbakop2021inverse,
  title={Inverse dynamics model-based shape control of soft continuum finger robot using parametric curve},
  author={Mbakop, Steeve and Tagne, Gilles and Frouin, Marc-Henri and Melingui, Achille and Merzouki, Rochdi},
  journal={IEEE Robotics and Automation Letters},
  volume={6},
  number={4},
  pages={8053--8060},
  year={2021},
  publisher={IEEE}
}

@article{della2019exact,
  title={Exact task execution in highly under-actuated soft limbs: an operational space based approach},
  author={Della Santina, Cosimo and Pallottino, Lucia and Rus, Daniela and Bicchi, Antonio},
  journal={IEEE Robotics and Automation Letters},
  volume={4},
  number={3},
  pages={2508--2515},
  year={2019},
  publisher={IEEE}
}

@article{li2022equivalent,
  title={Equivalent-input-disturbance-based dynamic tracking control for soft robots via reduced-order finite-element models},
  author={Li, Shijie and Kruszewski, Alexandre and Guerra, Thierry-Marie and Nguyen, Anh-Tu},
  journal={IEEE/ASME Transactions on Mechatronics},
  volume={27},
  number={5},
  pages={4078--4089},
  year={2022},
  publisher={IEEE}
}

@article{boyer2006macro,
  title={Macro-continuous computed torque algorithm for a three-dimensional eel-like robot},
  author={Boyer, Fr{\'e}d{\'e}ric and Porez, Mathieu and Khalil, Wisama},
  journal={IEEE transactions on robotics},
  volume={22},
  number={4},
  pages={763--775},
  year={2006},
  publisher={IEEE}
}

@article{boyer2011macrocontinuous,
  title={Macrocontinuous dynamics for hyperredundant robots: application to kinematic locomotion bioinspired by elongated body animals},
  author={Boyer, Fr{\'e}d{\'e}ric and Ali, Shaukat and Porez, Mathieu},
  journal={IEEE Transactions on Robotics},
  volume={28},
  number={2},
  pages={303--317},
  year={2011},
  publisher={IEEE}
}

@article{della2023model,
  title={Model-based control of soft robots: A survey of the state of the art and open challenges},
  author={Della Santina, Cosimo and Duriez, Christian and Rus, Daniela},
  journal={IEEE Control Systems Magazine},
  volume={43},
  number={3},
  pages={30--65},
  year={2023},
  publisher={IEEE}
}

@article{renda2018discrete,
  title={Discrete cosserat approach for multisection soft manipulator dynamics},
  author={Renda, Federico and Boyer, Fr{\'e}d{\'e}ric and Dias, Jorge and Seneviratne, Lakmal},
  journal={IEEE Transactions on Robotics},
  volume={34},
  number={6},
  pages={1518--1533},
  year={2018},
  publisher={IEEE}
}

@article{george2020first,
  title={First-order dynamic modeling and control of soft robots},
  author={George Thuruthel, Thomas and Renda, Federico and Iida, Fumiya},
  journal={Frontiers in Robotics and AI},
  volume={7},
  pages={95},
  year={2020},
  publisher={Frontiers Media SA}
}

@article{renda2022geometrically,
  title={Geometrically-exact inverse kinematic control of soft manipulators with general threadlike actuators’ routing},
  author={Renda, Federico and Armanini, Costanza and Mathew, Anup and Boyer, Frederic},
  journal={IEEE Robotics and Automation Letters},
  volume={7},
  number={3},
  pages={7311--7318},
  year={2022},
  publisher={IEEE}
}

@article{hartley2020contact,
  title={Contact-aided invariant extended Kalman filtering for robot state estimation},
  author={Hartley, Ross and Ghaffari, Maani and Eustice, Ryan M and Grizzle, Jessy W},
  journal={The International Journal of Robotics Research},
  volume={39},
  number={4},
  pages={402--430},
  year={2020},
  publisher={SAGE Publications Sage UK: London, England}
}

@inproceedings{tully2011shape,
  title={Shape estimation for image-guided surgery with a highly articulated snake robot},
  author={Tully, Stephen and Kantor, George and Zenati, Marco A and Choset, Howie},
  booktitle={2011 IEEE/RSJ International Conference on Intelligent Robots and Systems},
  pages={1353--1358},
  year={2011},
  organization={IEEE}
}

@inproceedings{ataka2016real,
  title={Real-time pose estimation and obstacle avoidance for multi-segment continuum manipulator in dynamic environments},
  author={Ataka, Ahmad and Qi, Peng and Shiva, Ali and Shafti, Ali and Wurdemann, Helge and Liu, Hongbin and Althoefer, Kaspar},
  booktitle={2016 IEEE/RSJ International Conference on Intelligent Robots and Systems (IROS)},
  pages={2827--2832},
  year={2016},
  organization={IEEE}
}

@article{navarro2020model,
  title={A model-based sensor fusion approach for force and shape estimation in soft robotics},
  author={Navarro, Stefan Escaida and Nagels, Steven and Alagi, Hosam and Faller, Lisa-Marie and Goury, Olivier and Morales-Bieze, Thor and Zangl, Hubert and Hein, Bj{\"o}rn and Ramakers, Raf and Deferme, Wim and others},
  journal={IEEE Robotics and Automation Letters},
  volume={5},
  number={4},
  pages={5621--5628},
  year={2020},
  publisher={IEEE}
}

@article{bruder2020data,
  title={Data-driven control of soft robots using Koopman operator theory},
  author={Bruder, Daniel and Fu, Xun and Gillespie, R Brent and Remy, C David and Vasudevan, Ram},
  journal={IEEE Transactions on Robotics},
  volume={37},
  number={3},
  pages={948--961},
  year={2020},
  publisher={IEEE}
}

@article{thuruthel2018model,
  title={Model-based reinforcement learning for closed-loop dynamic control of soft robotic manipulators},
  author={Thuruthel, Thomas George and Falotico, Egidio and Renda, Federico and Laschi, Cecilia},
  journal={IEEE Transactions on Robotics},
  volume={35},
  number={1},
  pages={124--134},
  year={2018},
  publisher={IEEE}
}

@article{gu2017survey,
  title={A survey on dielectric elastomer actuators for soft robots},
  author={Gu, Guo-Ying and Zhu, Jian and Zhu, Li-Min and Zhu, Xiangyang},
  journal={Bioinspiration \& biomimetics},
  volume={12},
  number={1},
  pages={011003},
  year={2017},
  publisher={IOP Publishing}
}

@inproceedings{duriez2013control,
  title={Control of elastic soft robots based on real-time finite element method},
  author={Duriez, Christian},
  booktitle={2013 IEEE international conference on robotics and automation},
  pages={3982--3987},
  year={2013},
  organization={IEEE}
}

@article{goury2018fast,
  title={Fast, generic, and reliable control and simulation of soft robots using model order reduction},
  author={Goury, Olivier and Duriez, Christian},
  journal={IEEE Transactions on Robotics},
  volume={34},
  number={6},
  pages={1565--1576},
  year={2018},
  publisher={IEEE}
}

@article{armanini2023soft,
  title={Soft robots modeling: A structured overview},
  author={Armanini, Costanza and Boyer, Fr{\'e}d{\'e}ric and Mathew, Anup Teejo and Duriez, Christian and Renda, Federico},
  journal={IEEE Transactions on Robotics},
  volume={39},
  number={3},
  pages={1728--1748},
  year={2023},
  publisher={IEEE}
}

@ARTICLE{9057619,
  author={Renda, Federico and Armanini, Costanza and Lebastard, Vincent and Candelier, Fabien and Boyer, Frederic},
  journal={IEEE Robotics and Automation Letters}, 
  title={A Geometric Variable-Strain Approach for Static Modeling of Soft Manipulators With Tendon and Fluidic Actuation}, 
  year={2020},
  volume={5},
  number={3},
  pages={4006-4013},
  doi={10.1109/LRA.2020.2985620}}

@book{shabana2020dynamics,
  title={Dynamics of multibody systems},
  author={Shabana, Ahmed A},
  year={2020},
  publisher={Cambridge university press}
}

@article{mathew2025reduced,
  title={Reduced order modeling of hybrid soft-rigid robots using global, local, and state-dependent strain parameterization},
  author={Mathew, Anup Teejo and Feliu-Talegon, Daniel and Alkayas, Abdulaziz Y and Boyer, Frederic and Renda, Federico},
  journal={The International Journal of Robotics Research},
  volume={44},
  number={1},
  pages={129--154},
  year={2025},
  publisher={SAGE Publications Sage UK: London, England}
}
%
%

\end{document}